\documentclass[12pt]{article}
\usepackage{amsmath}
\usepackage{graphicx}
\usepackage[hidelinks]{hyperref}
\usepackage{natbib}
\usepackage{amsthm}
\usepackage{amssymb}
\usepackage{algorithm}
\usepackage{algpseudocode}
\usepackage{booktabs}
\usepackage{enumitem}
\usepackage{bm}
\usepackage{multirow}
\usepackage[margin=1in]{geometry}
\usepackage{tikz}
\usepackage{pgfplots}
\pgfplotsset{compat=1.17}
\usetikzlibrary{arrows.meta, positioning, calc, patterns, decorations.pathreplacing, backgrounds, fit, shapes.geometric, matrix}

\newcommand{\R}{\mathbb{R}}
\newcommand{\Z}{\mathbb{Z}}
\newcommand{\E}{\mathbb{E}}
\newcommand{\norm}[1]{\left\lVert #1 \right\rVert}
\newcommand{\argsort}{\mathrm{argsort}}
\newcommand{\orderdesc}{\mathrm{argsort}_{\downarrow}}
\newcommand{\rankop}{\mathrm{rank}}

\newtheorem{theorem}{Theorem}
\newtheorem{lemma}[theorem]{Lemma}
\newtheorem{corollary}[theorem]{Corollary}
\newtheorem{proposition}[theorem]{Proposition}
\newtheorem{definition}[theorem]{Definition}
\newtheorem{remark}{Remark}

\begin{document}

\title{ArrowFlow: Hierarchical Machine Learning\\in the Space of Permutations}
\author{Ozgur Yilmaz\footnote{ozguryilmaz@atu.edu.tr; \url{https://scholar.google.com/citations?user=AIJWYCAAAAAJ}} \\
\small Department of Artificial Intelligence\\
\small Adana Science and Technology University, Adana, Turkey}
\date{}
\maketitle

\begin{abstract}
We introduce \textbf{ArrowFlow}, a machine learning architecture whose core layers operate entirely in the space of permutations (real-valued inputs are first encoded into permutations by a preprocessing pipeline). Its computational units are \emph{ranking filters}---learned orderings that compare inputs via Spearman's footrule distance and update through permutation-matrix accumulation, a non-gradient rule rooted in displacement evidence. Layers compose hierarchically: each layer's output ranking becomes the next layer's input, enabling deep ordinal representation learning without any floating-point parameters in the core computation.

An encoding pipeline (polynomial expansion, random projection, argsort) bridges real-valued data and the permutation world, while a multi-view ensemble---independent networks trained on diverse projections, combined by majority vote---compensates for information lost in the ordinal encoding. We connect the architecture to Arrow's impossibility theorem: the rank aggregation performed by its learning rule provably violates the independence-of-irrelevant-alternatives axiom, and we argue that such social-choice non-neutrality (context dependence, specialization, symmetry breaking) acts as an inductive bias for nonlinearity, sparsity, and stability.

Experiments span UCI tabular benchmarks, MNIST, gene expression cancer classification (TCGA), and preference data, all against GridSearchCV-tuned baselines. ArrowFlow matches the best tuned baselines on Iris (2.7\% vs.\ 3.3\%; a sub-sample margin, and its configuration was selected on the test split---a proof-of-concept rather than a fair head-to-head, see protocol) and is competitive on most UCI datasets. Against a nearest-neighbor classifier on the \emph{same} encoded permutations, the learned filters cut error by 2.5--12.5$\times$, evidence that the learning rule contributes beyond the encoding. On gene expression data, rank-based methods are \emph{exactly invariant} to strictly monotone within-sample transformations (stable at 2.5\% where SVM on raw values collapses to 82.6\%)---a property of the ordinal representation that ArrowFlow embodies architecturally rather than a unique advantage of the method. On MNIST via PCA, ArrowFlow reaches 9.1\% error using only ordinal comparisons, with strong scaling as network width and depth increase. A single parameter---polynomial degree---acts as a master switch: degree 1 yields noise robustness (8--28\% less degradation), rank-only (magnitude-hiding) robustness (+0.5pp cost), and missing-feature resilience; higher degrees trade these for improved clean accuracy.

ArrowFlow is not designed to surpass gradient-based methods. It is an existence proof that competitive classification is possible in a fundamentally different computational paradigm---one that elevates ordinal structure to a first-class citizen, with natural alignment to integer-only and neuromorphic hardware.

\medskip
\noindent\textbf{Code:} \url{https://github.com/yilmazozgur/arrowflow}
\end{abstract}

\section{Introduction}
\label{sec:intro}

The dominant paradigm in modern machine learning represents data as real-valued tensors and learns transformations through gradient descent on differentiable, continuous parameter spaces. While extraordinarily successful, this paradigm has fundamental limitations when the underlying structure of data is \emph{ordinal} or \emph{relational} rather than metric. Consider classifying sequences by their ordering pattern: the sequence $[1, 2, 3, 4, 5]$ is ``ascending'' and $[5, 4, 3, 2, 1]$ is ``descending.'' The essential information is not the magnitude of elements but their \emph{relative ordering}---a combinatorial, not metric, property.

ArrowFlow is built on the premise that \textbf{the fundamental data structure in a neural network can be a sorted list, and the fundamental operation can be a rank distance between two such lists}. In this framework:

\begin{itemize}[leftmargin=2em]
  \item \textbf{Data points} are ordered sequences of tokens (permutations of a vocabulary).
  \item \textbf{Learned filters} are also permutations of the same vocabulary.
  \item \textbf{Similarity} is measured by Spearman's footrule distance---the sum of absolute positional displacements between two orderings \citep{diaconis1977footrule}.
  \item \textbf{Learning} proceeds by re-ordering filter elements based on accumulated displacement evidence, rather than by adjusting floating-point weights.
\end{itemize}

This architecture draws from a productive tension in social choice theory. Arrow's impossibility theorem \citep{arrow1951social} demonstrates that any aggregator of individual rankings must violate at least one fairness axiom (Pareto efficiency, independence of irrelevant alternatives, non-dictatorship) on some preference profiles. In learning, we \emph{exploit} these violations as inductive biases: context dependence (IIA violations) yields nonlinearity; specialization (ND violations) yields sparsity and winner-take-all dynamics; partial monotonicity (Pareto-like behavior) stabilizes training. The name ``ArrowFlow'' reflects both this connection to Arrow's theorem and the flow of ordering information through the network layers.

\paragraph{What ArrowFlow is---and is not.}
ArrowFlow is not designed to be the most accurate or fastest classifier. It is a \emph{fundamentally different computational paradigm}: a hierarchical, filter-based machine learning system whose core layers operate entirely in the space of permutations (with a real-valued encoder only at the input). The contribution is the paradigm itself---showing that competitive classification is achievable without any floating-point parameters in the core computation, and that ordinal architectures possess structural advantages (noise robustness, rank-only/magnitude-hiding operation, graceful handling of missing data) that are not automatic in conventional methods unless rank transformations or robust preprocessing are added.

\paragraph{Contributions.}
\begin{enumerate}[leftmargin=2em]
  \item A formal definition of the \emph{ranking layer} with displacement motions, permutation-matrix accumulation, and hierarchical rank aggregation (Section~\ref{sec:ranking-layer}).
  \item The multi-view ensemble architecture with diverse projection strategies that addresses the argsort encoding bottleneck and achieves competitive accuracy (Sections~\ref{sec:encoding}--\ref{sec:ensemble}).
  \item An Arrow-theoretic interpretation of layer expressivity, connecting fairness-axiom violations to inductive biases for nonlinearity, sparsity, and stability (Section~\ref{sec:arrow}).
  \item Formal theoretical analysis (Section~\ref{sec:theory}):
  \begin{itemize}[leftmargin=1.4em,topsep=2pt]
    \item an argsort stability theorem with Gaussian tail bound, and a projected-noise corollary, explaining the noise robustness;
    \item an information-capacity theorem ($\log_2(e!)$ bits for an $e$-dimensional score vector, via permutation cones);
    \item a polynomial noise-amplification (local sensitivity) result for the accuracy--robustness trade-off;
    \item a consistency proof for the accumulation rule---a Borda/mean-position estimator---and an ensemble error bound; and
    \item a social-choice subsection characterizing filter learning as consistent rank aggregation (with exponential consistency), proving two-sided adversarial-perturbation bounds from the footrule metric, and showing that the learning rule's positional aggregation provably violates Arrow's independence axiom.
  \end{itemize}
  \item Comprehensive experiments across UCI tabular benchmarks, MNIST (via PCA, isolating the sort-layer classifier), gene expression cancer classification (TCGA), and the Sushi preference dataset. These reveal a unique robustness profile governed by the polynomial degree parameter, exact invariance to within-sample strictly monotone distribution shifts, and architecture scaling that has not plateaued at 1024 filters (Sections~\ref{sec:experiments}--\ref{sec:utility}).
  \item A quantitative energy analysis showing that ArrowFlow's integer-only arithmetic is 15$\times$ cheaper \emph{per layer in arithmetic energy} than equivalent FP32 MLP layers (narrowing to $\approx$1.05$\times$ at full seven-view inference, with larger gains expected from reduced memory traffic and multiplier-free hardware), with natural alignment to neuromorphic hardware and sorting-network implementations (Section~\ref{sec:energy}).
\end{enumerate}

\paragraph{Note.} This is an extended technical report presenting the full theory, experiments, and discussion for ArrowFlow.

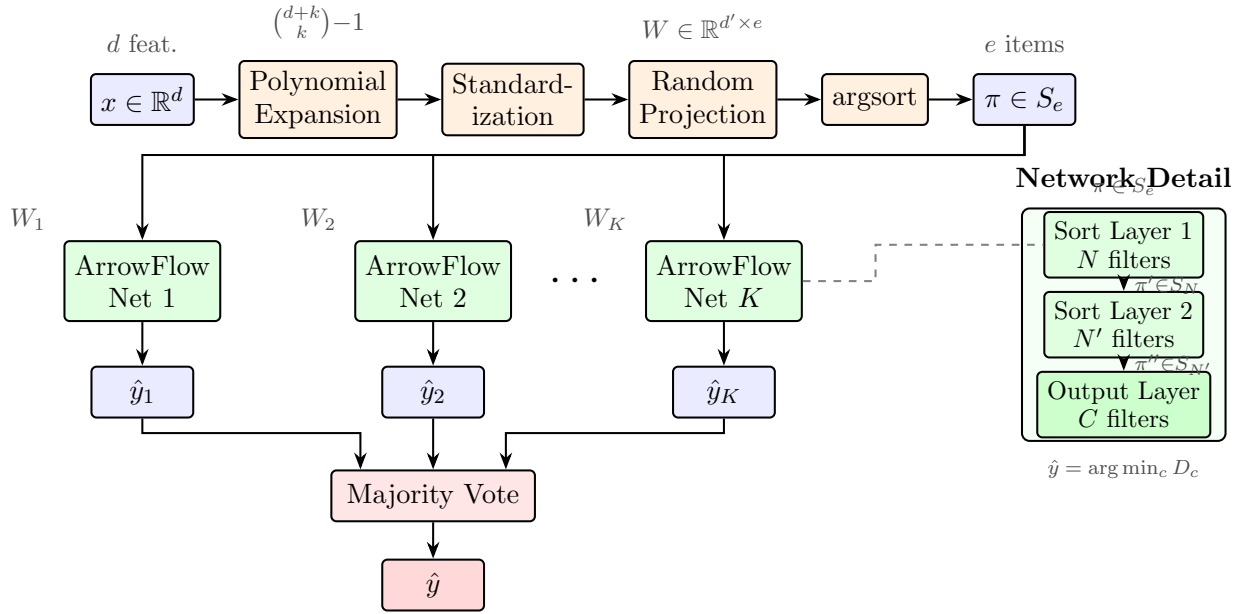
\begin{figure}[t]
\centering
\resizebox{\textwidth}{!}{%
\begin{tikzpicture}[
    >=Stealth,
    box/.style={draw, rounded corners=2pt, minimum height=0.7cm, minimum width=1.4cm,
                font=\small, align=center, thick},
    databox/.style={box, fill=blue!8},
    procbox/.style={box, fill=orange!12},
    sortbox/.style={box, fill=green!12},
    votebox/.style={box, fill=red!10},
    arr/.style={->, thick},
    lbl/.style={font=\footnotesize, text=black!70},
]
\node[databox] (input) at (0,0) {$x \in \R^d$};
\node[procbox, right=0.6 of input] (poly) {Polynomial\\[-1pt]Expansion};
\node[procbox, right=0.6 of poly] (std) {Standard-\\[-1pt]ization};
\node[procbox, right=0.6 of std] (proj) {Random\\[-1pt]Projection};
\node[procbox, right=0.6 of proj] (args) {$\argsort$};
\node[databox, right=0.6 of args] (perm) {$\pi \in S_e$};
\draw[arr] (input) -- (poly);
\draw[arr] (poly) -- (std);
\draw[arr] (std) -- (proj);
\draw[arr] (proj) -- (args);
\draw[arr] (args) -- (perm);
\node[lbl, above=0.15 of input] {$d$ feat.};
\node[lbl, above=0.15 of poly] {$\binom{d+k}{k}{-}1$};
\node[lbl, above=0.15 of proj] {$W \in \R^{d'\times e}$};
\node[lbl, above=0.15 of perm] {$e$ items};
\node[sortbox, minimum height=1.1cm] (net1) at (0, -2.5) {ArrowFlow\\[-1pt]Net 1};
\node[sortbox, minimum height=1.1cm] (net2) at (4, -2.5) {ArrowFlow\\[-1pt]Net 2};
\node[sortbox, minimum height=1.1cm] (netK) at (8, -2.5) {ArrowFlow\\[-1pt]Net $K$};
\node at ($(net2)!0.5!(netK)$) {\Large$\cdots$};
\node[lbl, above left=0.0 and 0.1 of net1] {$W_1$};
\node[lbl, above left=0.0 and 0.1 of net2] {$W_2$};
\node[lbl, above left=0.0 and 0.1 of netK] {$W_K$};
\draw[arr] (perm.south) -- ++(0,-0.4) -| (net1.north);
\draw[arr] (perm.south) -- ++(0,-0.4) -| (net2.north);
\draw[arr] (perm.south) -- ++(0,-0.4) -| (netK.north);
\node[databox, below=0.6 of net1] (pred1) {$\hat{y}_1$};
\node[databox, below=0.6 of net2] (pred2) {$\hat{y}_2$};
\node[databox, below=0.6 of netK] (predK) {$\hat{y}_K$};
\draw[arr] (net1) -- (pred1);
\draw[arr] (net2) -- (pred2);
\draw[arr] (netK) -- (predK);
\node[votebox, below=0.7 of pred2, minimum width=2.8cm] (vote) {Majority Vote};
\draw[arr] (pred1.south) -- ++(0,-0.2) -| (vote.160);
\draw[arr] (pred2) -- (vote);
\draw[arr] (predK.south) -- ++(0,-0.2) -| (vote.20);
\node[databox, below=0.5 of vote, fill=red!15] (final) {$\hat{y}$};
\draw[arr] (vote) -- (final);
\begin{scope}[shift={(13.5,-1.5)}]
\node[font=\small\bfseries, anchor=south] at (0,0.5) {Network Detail};
\draw[thick, rounded corners=3pt, fill=green!4] (-1.4,-3.2) rectangle (1.4,0);
\node[sortbox, minimum width=2.0cm, font=\footnotesize] (L1) at (0,-0.5) {Sort Layer 1\\[-1pt]$N$ filters};
\node[sortbox, minimum width=2.0cm, font=\footnotesize] (L2) at (0,-1.6) {Sort Layer 2\\[-1pt]$N'$ filters};
\node[sortbox, minimum width=2.0cm, fill=green!20, font=\footnotesize] (Lout) at (0,-2.7) {Output Layer\\[-1pt]$C$ filters};
\draw[arr] (L1) -- node[lbl, right, font=\scriptsize] {$\pi'{\in}S_N$} (L2);
\draw[arr] (L2) -- node[lbl, right, font=\scriptsize] {$\pi''{\in}S_{N'}$} (Lout);
\node[lbl, above=0.08 of L1, font=\scriptsize] {$\pi \in S_e$};
\node[lbl, below=0.15 of Lout, font=\scriptsize] {$\hat{y} = \arg\min_c D_c$};
\end{scope}
\draw[dashed, gray, thick] (netK.east) -- ++(1.0,0) |- (L1.west);
\end{tikzpicture}%
}
\caption{\textbf{ArrowFlow pipeline.} Real-valued input $x$ is encoded via polynomial expansion, standardization, random projection, and argsort. $K$ independent networks receive different projections $W_k$. Each network (detail, right) is a stack of sort layers computing Spearman's footrule distance to learned filters. Predictions are combined by majority vote.}
\label{fig:pipeline}
\end{figure}

\section{Related Work}
\label{sec:related}

ArrowFlow sits at the intersection of several research traditions: neuroscience-inspired rank-order coding, differentiable sorting, permutation distance metrics, non-gradient learning, social choice theory, random projections, and ensemble methods. We survey each in turn, positioning ArrowFlow's contributions relative to prior work.

\subsection{Rank Order Coding}
Rank Order Coding (ROC) encodes stimulus strength via the \emph{order} of neural spikes rather than absolute firing rates, highlighting order as an information carrier \citep{thorpe1998rapid, gautrais1998rate}. This biological principle---that relative ordering is more robust and efficient than precise magnitude encoding---motivates ArrowFlow's core design. The approach has shown promise in image recognition, where relative pixel intensities can be more informative than precise values under varying conditions \citep{jost2002rank, bonilla2022analyzing}. ArrowFlow generalizes ROC from a fixed encoding scheme to a \emph{learnable hierarchical transformation} of orderings across layers.

\subsection{Differentiable Sorting and Sinkhorn Networks}
Recent work has made sorting operations differentiable for end-to-end learning. NeuralSort \citep{grover2019neuralsort} and SoftSort \citep{prillo2020softsort} provide continuous relaxations of argsort. Blondel et al.\ \citep{blondel2020fast} introduced fast differentiable sorting via optimal transport. Gumbel-Sinkhorn networks \citep{mena2018gumbelsinkhorn} learn latent permutations through doubly stochastic matrices, leveraging the classical Sinkhorn normalization \citep{sinkhorn1967diagonal}. These approaches treat order as a soft surrogate for continuous scores---they relax permutations everywhere to enable gradient flow. ArrowFlow takes the opposite approach: it \textbf{operates directly on discrete permutations} and replaces gradient descent with permutation-matrix-based updates. Differentiability is needed only at the boundaries between rank and tensor layers, not within the ranking computation itself.

\subsection{Permutation Distance Metrics}
Kendall \citep{kendall1938} introduced the tau rank correlation coefficient, counting pairwise disagreements. Diaconis and Graham \citep{diaconis1977footrule} established Spearman's footrule as a closely related metric, proving that footrule and Kendall tau distance are within a constant factor. ArrowFlow uses Spearman's footrule ($\ell_1$ norm of positional displacements) as its core distance. Fagin et al.\ \citep{fagin2003topk} extended these metrics to partial rankings, relevant to ArrowFlow's handling of receptive fields and variable-length inputs.

\subsection{Non-Gradient Learning}
ArrowFlow's sort layers do not use gradient descent. Its closest relative is \emph{Learning Vector Quantization} (LVQ) and Kohonen's Self-Organizing Maps \citep{kohonen1990self}, which learn by competitive prototype updates---the nearest prototype is nudged toward (LVQ also: away from) each labeled input. ArrowFlow's output layer is precisely this attractive, error-driven prototype rule, and its hidden layers add the LVQ-style repulsive direction (Section~\ref{sec:supervised-training}), with two differences from classical LVQ: the prototypes are \emph{permutations} updated by re-ordering rather than real-valued vectors updated by subtraction, and learning composes hierarchically across stacked layers. This competitive dynamic also relates to Grossberg's Adaptive Resonance Theory \citep{grossberg1987competitive}, where winner-take-all prototypes are updated by matching and accumulation. Pointer Networks \citep{vinyals2015pointernetworks} and the ``Order Matters'' framework \citep{vinyals2015order} demonstrated that input and output ordering significantly impacts sequence-to-sequence performance, motivating architectures that treat ordering as a learnable structure. More recently, Hinton's Forward-Forward algorithm \citep{hinton2022forward} and Target Propagation \citep{lee2015difference} have revived interest in alternatives to backpropagation. ArrowFlow contributes to this line by demonstrating that competitive, displacement-based learning can work in discrete permutation spaces.

\subsection{Social Choice Theory and Rank Aggregation}
Arrow's impossibility theorem \citep{arrow1951social} demonstrates unavoidable trade-offs in rank aggregation---no mechanism aggregating three or more alternatives can simultaneously satisfy Pareto efficiency, independence of irrelevant alternatives, and non-dictatorship. Rank aggregation has been extensively studied in information retrieval, where Dwork et al.\ \citep{dwork2001rank} proposed combining ranked results from multiple search engines using median-based methods on Kendall tau distance. The Mallows model \citep{mallows1957nonnull} provides a probabilistic framework for permutations centered on a modal ranking, parameterized by a dispersion parameter. Conitzer and Sandholm \citep{conitzer2005voting} showed that common voting rules (Borda, Kemeny) correspond to maximum likelihood estimation under different noise models over permutations, providing a statistical-learning interpretation of the social-choice concepts that ArrowFlow operationalizes. Rank-Ordered Autoencoders \citep{bertens2016rank} use rank-preserving embeddings and sparsity constraints. ArrowFlow generalizes from \emph{preserving} rank to \emph{transforming} rank across layers, and uniquely reinterprets Arrow's fairness-axiom violations as inductive biases for nonlinearity, sparsity, and stability.

\subsection{Random Projections and Dimensionality Reduction}
ArrowFlow's encoding pipeline relies on random projection to map data into a space where argsort produces informative permutations. The Johnson-Lindenstrauss lemma \citep{johnson1984extensions} guarantees that random projections approximately preserve pairwise distances. Achlioptas \citep{achlioptas2003database} showed that even very sparse random matrices suffice for this purpose. Arriaga and Vempala \citep{arriaga2006algorithmic} proved that random projections preserve margin-based learnability, directly justifying ArrowFlow's pipeline of random projection followed by classification. Cannings and Samworth \citep{cannings2017random} proposed random-projection ensemble classification, combining classifiers trained on different random projections via data-driven voting---ArrowFlow's multi-view ensemble is a permutation-space instance of exactly this framework.

\subsection{Ensemble Methods and Multi-View Learning}
ArrowFlow's multi-view ensemble draws on classical ensemble theory. Breiman's bagging \citep{breiman1996bagging} established variance reduction through bootstrap aggregation. Schapire \citep{schapire1990strength} proved that any weak learner can be boosted to arbitrary accuracy, providing theoretical grounding for why ArrowFlow's individual views, even if weak, combine into a strong ensemble. Dietterich \citep{dietterich2000ensemble} showed that ensemble diversity is as important as individual accuracy. Condorcet's jury theorem provides the theoretical foundation: if each voter has error rate $p < 0.5$ and errors are independent, the majority-vote error decreases exponentially with the number of voters. ArrowFlow maximizes the independence condition by using diverse projection strategies---each view sees a different ordinal representation of the same data. This connects to the multi-view learning framework of Xu et al.\ \citep{xu2013survey}, where complementary ``views'' are fused for improved prediction.

\subsection{Robustness Through Discretization}
A distinctive finding of ArrowFlow is that the argsort discretization provides measurable noise robustness when polynomial expansion is absent. This connects to a broader literature on discretization for robustness. Quantization-aware training \citep{jacob2018quantization} shows that reduced-precision representations can improve robustness. The VQ-VAE \citep{oord2017neural} demonstrated that discretizing latent representations via vector quantization improves generalization and avoids posterior collapse---ArrowFlow's argsort is an extreme form of discretization that converts continuous features to combinatorial objects. More directly, the rank transform \citep{conover1981rank} has long been used in nonparametric statistics precisely because it is immune to monotone transformations and outliers. ArrowFlow's encoding pipeline makes this classical statistical insight into an architectural principle.

\subsection{Learning to Rank}
ArrowFlow also connects to the learning-to-rank literature in information retrieval, where the goal is to learn a function that orders items by relevance. RankNet \citep{burges2005learning} uses gradient descent on pairwise ranking losses; LambdaMART \citep{burges2010ranknet} combines boosted trees with ranking-specific objectives. These methods learn continuous scoring functions whose \emph{output} is a ranking. ArrowFlow inverts this relationship: both the \emph{input} and the \emph{learned parameters} are rankings, and learning operates by re-ordering filters rather than adjusting scores. This makes ArrowFlow more akin to a nearest-neighbor method in permutation metric space, enhanced with hierarchical learned embeddings.

\section{The Ranking Layer}
\label{sec:ranking-layer}

The ranking layer is the fundamental computational unit of ArrowFlow. It operates in the space of orderings rather than numeric activations. Each layer transforms an input ranking into a new ranking representation by comparing it to a bank of learned \emph{ranking filters}. We use the terms \emph{ranking layer} and \emph{sort layer} interchangeably throughout.

\subsection{Notation}
Let $\mathcal{V} = \{1, \dots, V\}$ denote a vocabulary of tokens. A ranking is a permutation $\pi \in S_V$. We use two complementary maps and keep them distinct throughout:
\[
\operatorname{item}_\pi(p) \;=\; \text{the token at position } p, \qquad
\operatorname{pos}_\pi(v) \;=\; \text{the position of token } v.
\]
Worked examples list tokens by position---e.g.\ $\pi_x = [C,A,E,B,D]$ means $\operatorname{item}_{\pi_x}(1)=C$---while the displacement formulas below use $\operatorname{pos}$. We write $\argsort(y)$ for the permutation that lists the entries of $y \in \R^V$ in \emph{increasing} order (matching NumPy's convention and the implementation); thus when the layer ranks filters by distance, the smallest distance---the closest filter---appears first. Where a decreasing order is more natural (the permutation cones of Theorem~\ref{thm:capacity}) we write $\orderdesc$; the two differ only by reversal. We abbreviate $\rankop(r, v) = \operatorname{pos}_r(v)$ for the position of token $v$ in a filter $r$'s ordering.

\subsection{Ranking Filters and Displacement Motions}
A ranking layer $F^{\mathrm{rank}}$ contains $N$ ranking filters $\{r_j\}_{j=1}^N$, each storing a learned local ordering over the vocabulary. Given an input ranking $\pi_x$, the layer computes a \emph{motion} vector (a signed positional displacement, not a string-edit operation) $m_j \in \Z^\ell$ for each filter:
\begin{equation}
m_j[p] = \rankop(r_j, \pi_x[p]) - p, \quad p = 1, \dots, \ell.
\label{eq:motion}
\end{equation}
The motion $m_j[p]$ measures how many positions element $\pi_x[p]$ must move to reach its position in filter $r_j$. This is a signed displacement---positive means the element should move toward the end, negative toward the front.

A distance is computed via the $\ell_q$ norm:
\begin{equation}
D_j = \norm{m_j}_q, \quad q \in \{1, 2\} \ \text{(with an optional $\ell_0$-style count, $q=0$, which is a pseudonorm)}.
\label{eq:distance}
\end{equation}
With $q = 1$, this is Spearman's footrule distance \citep{diaconis1977footrule}: the total positional displacement between the input and the filter. All experiments in this paper use $q=1$.

\paragraph{Missing data and receptive fields.}
If $\pi_x$ omits items present in $r_j$, the experiments in this paper assign each missing item the central position $mV$ (with $m=0.5$) and down-weight its update contribution by $w=0.5$ (Section~\ref{sec:supervised-training}); an alternative variant instead displaces missing items to the end of the list and charges a deletion penalty. We use the central-imputation rule throughout unless stated otherwise. If $r_j$ has a restricted receptive field (subset of items), the average motion is subtracted to promote translation invariance---analogous to zero-centering in convolutional layers.

\subsection{Forward Pass: From Distances to Output Ranking}
The layer aggregates $N$ filter responses into a new ranking. The simplest and most effective aggregation is:
\begin{equation}
\pi' = \argsort(D_1, D_2, \dots, D_N),
\end{equation}
which ranks filters from closest (smallest distance, most similar) to farthest. The output $\pi'$ is a permutation of filter indices---a new sorted list in a vocabulary of size $N$---and becomes the input to the next layer.

This is closer to a prototype (radial-basis) layer than to a convolution: each filter is a learned permutation prototype, the layer scores the input by rank distance to every prototype, and the resulting ranking of prototypes defines the next representation. (We reserve the convolution analogy for the receptive-field variant, where a filter sees only a subset of items.)

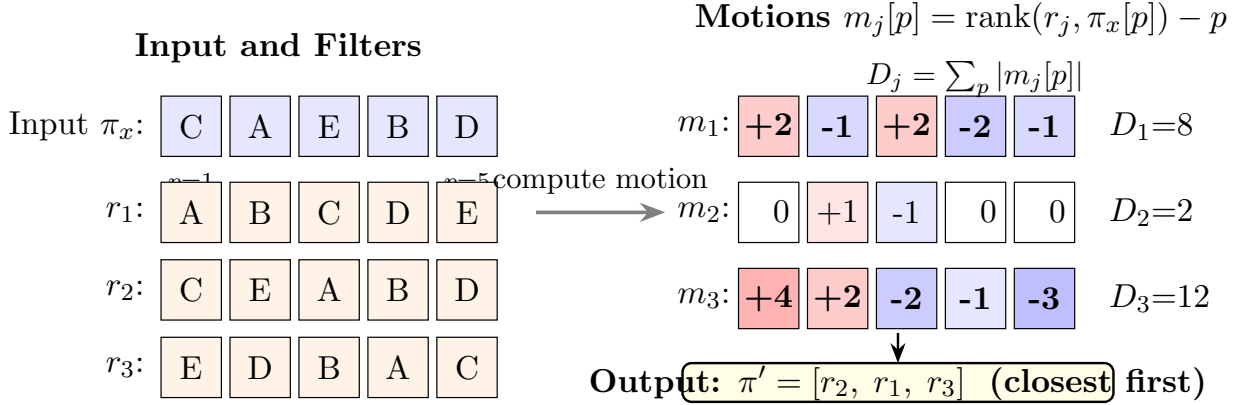
\begin{figure}[t]
\centering
\resizebox{\textwidth}{!}{%
\begin{tikzpicture}[
    >=Stealth,
    arr/.style={->, thick},
    lbl/.style={font=\small},
    cell/.style={draw, minimum width=0.7cm, minimum height=0.7cm, font=\small, inner sep=0pt},
    hcell/.style={cell, fill=blue!10},
    fcell/.style={cell, fill=orange!10},
    mcell/.style={cell, fill=green!10},
]
\node[font=\small\bfseries, anchor=south west] at (0.5, 1.8) {Input and Filters};
\node[lbl, anchor=east] at (0.9, 1.2) {Input $\pi_x$:};
\foreach \i/\v in {1/C, 2/A, 3/E, 4/B, 5/D} {
    \node[hcell] (in\i) at (0.5+\i*0.8, 1.2) {\v};
}
\node[lbl, below=0.04 of in1, font=\scriptsize] {$p{=}1$};
\node[lbl, below=0.04 of in5, font=\scriptsize] {$p{=}5$};
\node[lbl, anchor=east] at (0.9, 0.2) {$r_1$:};
\foreach \i/\v in {1/A, 2/B, 3/C, 4/D, 5/E} {
    \node[fcell] (f1\i) at (0.5+\i*0.8, 0.2) {\v};
}
\node[lbl, anchor=east] at (0.9, -0.7) {$r_2$:};
\foreach \i/\v in {1/C, 2/E, 3/A, 4/B, 5/D} {
    \node[fcell] (f2\i) at (0.5+\i*0.8, -0.7) {\v};
}
\node[lbl, anchor=east] at (0.9, -1.6) {$r_3$:};
\foreach \i/\v in {1/E, 2/D, 3/B, 4/A, 5/C} {
    \node[fcell] (f3\i) at (0.5+\i*0.8, -1.6) {\v};
}
\draw[arr, gray, very thick] (5.3, 0.2) -- (6.8, 0.2);
\node[font=\footnotesize, black, above=2pt] at (6.05, 0.2) {compute motion};
\node[font=\small\bfseries, anchor=south west] at (7.0, 2.1) {Motions $m_j[p] = \mathrm{rank}(r_j, \pi_x[p]) - p$};
\node[lbl, anchor=east] at (7.7, 1.2) {$m_1$:};
\foreach \i/\v/\c in {1/{+2}/red!20, 2/{-1}/blue!15, 3/{+2}/red!20, 4/{-2}/blue!20, 5/{-1}/blue!15} {
    \node[mcell, fill=\c, font=\small\bfseries] (m1\i) at (7.2+\i*0.8, 1.2) {\v};
}
\node[lbl, right=0.25 of m15] {$D_1{=}8$};
\node[lbl, anchor=east] at (7.7, 0.2) {$m_2$:};
\foreach \i/\v/\c in {1/{\phantom{+}0}/white, 2/{+1}/red!10, 3/{-1}/blue!10, 4/{\phantom{+}0}/white, 5/{\phantom{+}0}/white} {
    \node[mcell, fill=\c, font=\small] (m2\i) at (7.2+\i*0.8, 0.2) {\v};
}
\node[lbl, right=0.25 of m25] {$D_2{=}2$};
\node[lbl, anchor=east] at (7.7, -0.8) {$m_3$:};
\foreach \i/\v/\c in {1/{+4}/red!30, 2/{+2}/red!20, 3/{-2}/blue!20, 4/{-1}/blue!10, 5/{-3}/blue!25} {
    \node[mcell, fill=\c, font=\small\bfseries] (m3\i) at (7.2+\i*0.8, -0.8) {\v};
}
\node[lbl, right=0.25 of m35] {$D_3{=}12$};
\node[font=\footnotesize, anchor=east] at (11.8, 1.75) {$D_j = \sum_p |m_j[p]|$};
\draw[thick, rounded corners=3pt, fill=yellow!12] (7.0,-2.0) rectangle (12.0,-1.55);
\node[font=\small\bfseries] at (9.5,-1.78) {Output: $\pi' = [r_2,\; r_1,\; r_3]$ \;\small(closest first)};
\draw[arr, thick] (9.5, -1.2) -- (9.5, -1.55);
\end{tikzpicture}%
}
\caption{\textbf{Ranking layer forward pass.} Input $\pi_x = [C, A, E, B, D]$ is compared to three filters. The motion $m_j[p] = \mathrm{rank}(r_j, \pi_x[p]) - p$ measures signed displacement (red = forward, blue = backward). The footrule distance $D_j = \sum |m_j[p]|$ totals the displacements. Filters are ranked by proximity: $r_2$ (distance 2) is closest. The output $\pi'$ becomes the input to the next layer.}
\label{fig:forward-pass}
\end{figure}

\subsection{Backward Pass: Motion as Discrete Gradient}

The key insight is that the motion vector $m_j$ from Eq.~\eqref{eq:motion} plays the role of a gradient in continuous optimization---it tells us \emph{how to move each element} to reduce distance.

\begin{proposition}[Motion as Discrete Gradient]
\label{prop:motion-gradient}
Because the footrule distance $d_F(\pi_x, r_j) = \sum_{p=1}^\ell |m_j[p]|$ decomposes into independent per-position terms, its minimization over per-item integer shifts $\{\delta_p\}$ separates across positions, and the unique minimizer of each term is the motion of Eq.~\eqref{eq:motion}:
\[
m_j[p] = \arg\min_{\delta \in \Z}\; \big|\rankop(r_j,\, \pi_x[p]) - (p + \delta)\big|, \qquad p = 1, \dots, \ell.
\]
Hence the full motion vector $m_j$ is the exact per-item descent direction for the footrule objective under an independent-shift relaxation (the simultaneous shifts need not compose into a valid permutation).
\end{proposition}

The motion vector thus points each item directly toward its target position in the filter, the discrete analogue of a negative loss gradient---no calculus is involved, and the update signal is simply the motion between the input ordering and the correct-class filter. The separability exploited here is precisely what distinguishes footrule from non-decomposable metrics such as Kendall tau: it is what lets the accumulation rule of Section~\ref{sec:perm-accum} treat each item's displacement independently. (We caution that summing the per-item minimizers does not itself yield a permutation; the motion is a descent \emph{direction}, realized as a valid ordering only after the accumulation-and-reorder projection of Eqs.~\ref{eq:accum-update}--\ref{eq:accum-reorder}.)

\subsection{Permutation-Matrix Accumulation}
\label{sec:perm-accum}

Each filter $r_j$ maintains an accumulator $A_j \in \R^{V \times V}$ (initialized to identity) that integrates displacement votes across training examples. Given input $\pi_x$, the \emph{vote matrix} $\Phi(\pi_x, r_j) \in \{0,1\}^{V \times V}$ records where each item appears in the input:
\begin{definition}[Vote Matrix]
\label{def:vote-matrix}
Each item $v$ in the vocabulary votes for its position in the input $\pi_x$:
\[
\Phi(\pi_x, r_j)[v,\, k] = \begin{cases} 1 & \text{if } k = \rankop(\pi_x, v), \\ 0 & \text{otherwise,} \end{cases}
\]
where $\rankop(\pi_x, v)$ denotes the position of item $v$ in the input ordering $\pi_x$.
Each row of $\Phi$ has exactly one nonzero entry: item $v$ votes for the position that the training data places it at.
\end{definition}

Equivalently, for each filter position $p$, the backward motion $m_j^{\mathrm{bwd}}[p] = \rankop(\pi_x, r_j[p]) - p$ measures where filter item $r_j[p]$ appears in the input, and the vote places $r_j[p]$ at position $p + m_j^{\mathrm{bwd}}[p]$. Note that this backward motion is the negative of the forward motion from Eq.~\eqref{eq:motion}: the forward motion measures displacement \emph{from input to filter}, while the backward motion measures displacement \emph{from filter to input}, providing the update direction.

The accumulator is updated additively:
\begin{equation}
\label{eq:accum-update}
A_j \leftarrow A_j + \Phi(\pi_x, r_j).
\end{equation}
After $T$ updates, $A_j[v, k]$ counts the total evidence that item $v$ belongs at position $k$. The filter's ordering is recomputed via weighted index averaging:
\begin{equation}
\label{eq:accum-reorder}
\hat{\imath}(v) = \frac{\sum_k k \cdot A_j[v, k]}{\sum_k A_j[v, k]}, \qquad
r_j \leftarrow \argsort\big(\{\hat{\imath}(v)\}_{v \in \mathcal{V}}\big).
\end{equation}
This resembles a projection onto the set of permutations: accumulate displacement evidence, then map to the ordering consistent with the aggregate evidence (we do not claim a formal proximal operator). The accumulator smooths individual motions and provides a form of momentum.\footnote{Eq.~\eqref{eq:accum-reorder} normalizes each item's score by its row sum $\sum_k A_j[v,k]$; the implementation normalizes by the corresponding column sum. The two coincide (and the analysis of Proposition~\ref{prop:convergence} applies exactly) whenever the accumulated votes form a permutation matrix, i.e.\ absent the position clipping and missing-item handling used for partial inputs.}

\begin{remark}
The learning rule has no learning rate in the continuous sense. Instead, the accumulator's history provides natural damping: recent motions contribute proportionally to the total accumulated evidence. An explicit learning rate $\eta$ can optionally scale the vote matrix $\Phi$ to control update speed.
\end{remark}

\begin{figure}[t]
\centering
\resizebox{0.95\textwidth}{!}{%
\begin{tikzpicture}[
    >=Stealth,
    arr/.style={->, thick},
    lbl/.style={font=\small},
    cell/.style={draw, minimum width=0.7cm, minimum height=0.7cm, font=\small, inner sep=0pt},
]
\node[font=\small\bfseries] at (1.4, 3.7) {(a) Accumulator $A_j$};
\node[font=\scriptsize] at (1.4, 3.3) {(identity init)};
\node[lbl, font=\scriptsize] at (0.35, 2.9) {1};
\node[lbl, font=\scriptsize] at (1.05, 2.9) {2};
\node[lbl, font=\scriptsize] at (1.75, 2.9) {3};
\node[lbl, font=\scriptsize] at (2.45, 2.9) {4};
\foreach \r/\item in {1/A, 2/B, 3/C, 4/D} {
    \node[lbl, anchor=east] at (-0.1, 2.5-\r*0.7) {\item};
    \foreach \c in {1,2,3,4} {
        \pgfmathtruncatemacro{\val}{\r==\c ? 1 : 0}
        \pgfmathtruncatemacro{\shade}{\r==\c ? 25 : 0}
        \node[cell, fill=blue!\shade] at (-0.35+\c*0.7, 2.5-\r*0.7) {\val};
    }
}
\node[font=\LARGE\bfseries] at (3.8, 1.4) {$+$};
\node[font=\small\bfseries] at (6.4, 3.7) {(b) Vote $\Phi$};
\node[font=\scriptsize] at (6.4, 3.3) {Input $[C,A,D,B]$, Filter $[A,B,C,D]$};
\node[lbl, font=\scriptsize] at (5.35, 2.9) {1};
\node[lbl, font=\scriptsize] at (6.05, 2.9) {2};
\node[lbl, font=\scriptsize] at (6.75, 2.9) {3};
\node[lbl, font=\scriptsize] at (7.45, 2.9) {4};
\foreach \r/\item/\votecol in {1/A/2, 2/B/4, 3/C/1, 4/D/3} {
    \node[lbl, anchor=east] at (4.9, 2.5-\r*0.7) {\item};
    \foreach \c in {1,2,3,4} {
        \pgfmathtruncatemacro{\val}{\c==\votecol ? 1 : 0}
        \pgfmathtruncatemacro{\shade}{\c==\votecol ? 25 : 0}
        \node[cell, fill=orange!\shade] at (4.65+\c*0.7, 2.5-\r*0.7) {\val};
    }
}
\node[lbl, text width=2.8cm, align=left, anchor=north west, font=\footnotesize] at (8.5, 2.5) {Each item votes\\for its position\\in the input};
\node[font=\LARGE\bfseries] at (3.8, -0.1) {$=$};
\draw[arr, very thick] (3.8, -0.5) -- (3.8, -1.3);
\node[font=\small\bfseries] at (1.4, -1.5) {(c) $A_j + \Phi$};
\node[lbl, font=\scriptsize] at (0.35, -1.9) {1};
\node[lbl, font=\scriptsize] at (1.05, -1.9) {2};
\node[lbl, font=\scriptsize] at (1.75, -1.9) {3};
\node[lbl, font=\scriptsize] at (2.45, -1.9) {4};
\node[lbl, anchor=east] at (-0.1, -2.3) {A};
\node[cell, fill=blue!15] at (0.35, -2.3) {1};
\node[cell, fill=orange!20] at (1.05, -2.3) {1};
\node[cell] at (1.75, -2.3) {0};
\node[cell] at (2.45, -2.3) {0};
\node[lbl, anchor=east] at (-0.1, -3.0) {B};
\node[cell] at (0.35, -3.0) {0};
\node[cell, fill=blue!15] at (1.05, -3.0) {1};
\node[cell] at (1.75, -3.0) {0};
\node[cell, fill=orange!20] at (2.45, -3.0) {1};
\node[lbl, anchor=east] at (-0.1, -3.7) {C};
\node[cell, fill=orange!20] at (0.35, -3.7) {1};
\node[cell] at (1.05, -3.7) {0};
\node[cell, fill=blue!15] at (1.75, -3.7) {1};
\node[cell] at (2.45, -3.7) {0};
\node[lbl, anchor=east] at (-0.1, -4.4) {D};
\node[cell] at (0.35, -4.4) {0};
\node[cell] at (1.05, -4.4) {0};
\node[cell, fill=orange!20] at (1.75, -4.4) {1};
\node[cell, fill=blue!15] at (2.45, -4.4) {1};
\draw[arr, very thick] (3.4, -3.3) -- node[above, font=\footnotesize] {$\frac{\sum_k k \cdot A[v,k]}{\sum_k A[v,k]}$} (5.2, -3.3);
\node[font=\small\bfseries] at (6.5, -1.5) {(d) Reorder};
\node[lbl, anchor=west] at (5.5, -2.3) {$\hat{\imath}(A) = 1.5$};
\node[lbl, anchor=west] at (5.5, -2.9) {$\hat{\imath}(B) = 3.0$};
\node[lbl, anchor=west] at (5.5, -3.5) {$\hat{\imath}(C) = 2.0$};
\node[lbl, anchor=west] at (5.5, -4.1) {$\hat{\imath}(D) = 3.5$};
\draw[arr, thick] (6.3, -4.4) -- (6.3, -4.9);
\node[draw, rounded corners=3pt, fill=green!12, font=\small, inner sep=5pt] at (6.3, -5.3) {$r_j \leftarrow [A, C, B, D]$};
\end{tikzpicture}%
}
\caption{\textbf{Permutation-matrix accumulation.} (a)~Accumulator starts as identity (encoding current filter $[A,B,C,D]$). (b)~Training input $[C,A,D,B]$: each item votes for its position in the input (Def.~\ref{def:vote-matrix}). (c)~Votes sum additively (Eq.~\ref{eq:accum-update}). (d)~Weighted average and argsort yield the updated filter $[A,C,B,D]$ (Eq.~\ref{eq:accum-reorder}), which has moved toward the training data. Convergence to the mean-position ranking is guaranteed (Proposition~\ref{prop:convergence}).}
\label{fig:accumulation}
\end{figure}
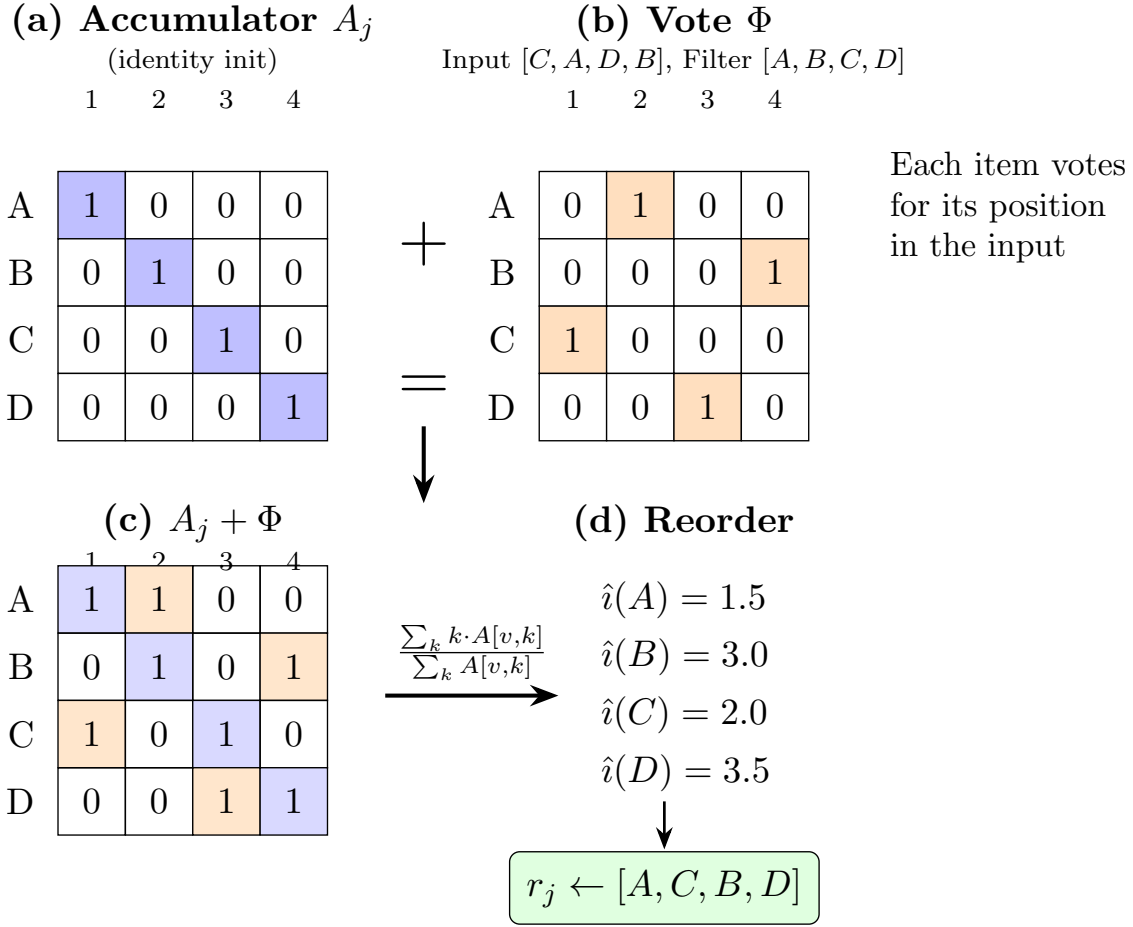

\subsection{Training Algorithm}

\begin{algorithm}[t]
\caption{Ranking Layer Forward--Backward}
\label{alg:rank-layer}
\begin{algorithmic}[1]
\Require Input ranking $\pi_x$, filters $\{r_j\}_{j=1}^N$, accumulators $\{A_j\}$
\Statex
\State \textbf{Forward:}
\For{$j = 1$ to $N$}
  \State $m_j \gets \mathrm{Motion}(\pi_x, r_j)$ \Comment{signed positional displacements, Eq.~\eqref{eq:motion}}
  \State $D_j \gets \|m_j\|_1$ \Comment{Spearman's footrule}
\EndFor
\State $\pi' \gets \argsort(D_1, \dots, D_N)$ \Comment{rank filters by proximity (closest first)}
\State \Return $\pi'$
\Statex
\State \textbf{Backward (supervised update):} see Section~\ref{sec:supervised-training} and Algorithm~\ref{alg:supervised}: the target ranking selected by the label (output layer) or propagated from above (hidden layers) is turned into a vote matrix $\Phi$ (Def.~\ref{def:vote-matrix}), accumulated into $A_j$ (Eqs.~\ref{eq:accum-update}--\ref{eq:accum-reorder}), and the filter reordered.
\end{algorithmic}
\end{algorithm}

\section{Arrow's Impossibility Theorem as Design Principle}
\label{sec:arrow}

Each ranking layer functions as a \emph{micro-society} of filters aggregating preferences. Arrow's impossibility theorem states that no social choice function over three or more alternatives can simultaneously satisfy:
\begin{itemize}[leftmargin=2em]
  \item \textbf{Pareto Efficiency (PE):} if all filters prefer $a \succ b$, the output preserves $a \succ b$.
  \item \textbf{Independence of Irrelevant Alternatives (IIA):} the relative output ranking of $a$ vs.\ $b$ depends only on filters' rankings of $a$ vs.\ $b$, not on $c$.
  \item \textbf{Non-Dictatorship (ND):} no single filter determines the output for all inputs.
\end{itemize}

In ArrowFlow, these constraints become \emph{capacity knobs}:

\paragraph{PE $\to$ Stability and gradient propagation.}
When filters unanimously prefer $a \succ b$, motions preserve this ordering, yielding a stabilizing monotonicity. This is analogous to residual connections---unanimous evidence propagates without attenuation.

\paragraph{IIA and context dependence.}
A single layer's \emph{forward} pass orders filters by their independent distances, so it respects pairwise IIA; its nonlinearity comes from the piecewise-constant cone geometry of $\argsort$ (Theorem~\ref{thm:capacity}), the permutation-space counterpart of an activation function. The IIA violation that Arrow's theorem guarantees appears instead in the \emph{learning rule}, whose positional aggregation of training rankings is context-dependent (Proposition~\ref{prop:iia-index}): an item's consensus position shifts with the placement of others. This is what shapes filters into context-sensitive detectors rather than pairwise-separable templates; the formal statement is deferred to Section~\ref{sec:social-choice-theory}.

\paragraph{ND violation $\to$ Sparse specialization.}
When a few filters consistently have the smallest distances for certain input types, they dominate the output ranking---a winner-take-all dynamic analogous to attention heads or sparse mixture-of-experts. This creates feature dominance, aiding generalization and efficiency.

\paragraph{Depth as compositional rank refinement.}
Stacking ranking layers enables compositional processing:
\begin{itemize}[leftmargin=2em]
  \item Early layers learn local tournaments or pairwise motifs (e.g., Condorcet-like cycles, single-peaked segments).
  \item Middle layers learn contextual re-weighting (via the context-dependent aggregation of Proposition~\ref{prop:iia-index}) to resolve conflicts.
  \item Late layers enforce global consistency and task-specific signals.
\end{itemize}
This is analogous to ConvNets (local filters $\to$ mid-level patterns $\to$ global structure), but in the permutation domain.

\section{Encoding: From Real-Valued Data to Permutations}
\label{sec:encoding}

The central challenge in applying ArrowFlow to real-valued data is the \emph{encoding problem}: converting continuous feature vectors into permutations without catastrophic information loss. The argsort operation that converts a real vector to a ranking discards all magnitude information, preserving only relative order. Two vectors $[1.0, 2.0, 3.0]$ and $[0.01, 100, 100.01]$ produce the same ranking $[1, 2, 3]$ despite being vastly different.

\subsection{Polynomial Feature Expansion}
For low-dimensional data, we first expand the feature space via polynomial features:
\begin{equation}
x \in \R^d \;\mapsto\; \phi(x) \in \R^{\binom{d+k}{k}-1},
\end{equation}
where $k$ is the polynomial degree. For $d = 4$ (Iris) at degree 3, this expands from 4 to 34 features, creating interaction terms $x_i x_j$, $x_i x_j x_k$ that dramatically increase the diversity of achievable rankings. This expansion is critical: on Iris, polynomial expansion alone reduces error by about 1.5$\times$ (21.3\%$\to$14.7\%; Table~\ref{tab:progression}), and contributes to a larger reduction in combination with the multi-view ensemble.

\subsection{Random Projection and Argsort}
After optional polynomial expansion and standardization, features are projected to a target embedding dimension $e$ via a random projection matrix $W \in \R^{d' \times e}$:
\begin{equation}
z = x \cdot W, \qquad \pi = \argsort(z).
\end{equation}
The resulting permutation $\pi$ is the input to the first sort layer. Different random matrices $W$ produce different permutations from the same input---the key source of ensemble diversity.

\subsection{Target-Aware Projection}
To inject supervised signal, we construct \emph{target-aware} projections that mix Linear Discriminant Analysis (LDA) components with random components:
\begin{equation}
W_{\mathrm{aware}} = \begin{bmatrix} W_{\mathrm{LDA}} & W_{\mathrm{random}} \end{bmatrix},
\end{equation}
where $W_{\mathrm{LDA}}$ captures the most class-discriminative directions and $W_{\mathrm{random}}$ provides diversity. The LDA ratio controls the balance.

\section{Multi-View Ensemble Architecture}
\label{sec:ensemble}

The most impactful architectural innovation in ArrowFlow is the multi-view ensemble. A single projection $W$ produces a single ``view'' of the data in permutation space. Since different projections capture different ordinal relationships, training multiple independent networks on different projections and combining via majority vote dramatically reduces error.

\subsection{Architecture}
Given $K$ views:
\begin{enumerate}
  \item Generate $K$ projection matrices $\{W_k\}_{k=1}^K$ using diverse strategies: \emph{random} (Section~\ref{sec:encoding}), \emph{target-aware} (LDA-augmented, Section~\ref{sec:encoding}), and \emph{calibrated}---a random projection whose projected scores are standardized per dimension before argsort, so that every projected dimension contributes equally to the ranking rather than being dominated by high-variance directions.
  \item For each view $k$: encode all data via $W_k$, train an independent ArrowFlow network.
  \item Combine predictions via majority vote: $\hat{y} = \mathrm{mode}(\hat{y}_1, \dots, \hat{y}_K)$.
\end{enumerate}

\subsection{Theoretical Foundation}
By Condorcet's jury theorem, if each view has error rate $p < 0.5$ and errors are independent, the ensemble error decreases exponentially with $K$ (formalized in Theorem~\ref{thm:ensemble}). Independence is maximized by diverse projection strategies: views cycle through target-aware, random, and calibrated projections. Empirically, 7 views provide the best accuracy--cost trade-off, reducing error by up to $\sim$3$\times$ compared to a single view (dataset-dependent; Section~\ref{sec:experiments}).

\subsection{Permutation Data Augmentation}
To improve sort-layer generalization, each training permutation can be augmented by applying random adjacent transpositions---the minimal perturbation in Spearman footrule distance (exactly 2 units). This is the permutation analogue of Gaussian noise in Euclidean space, adapted to the discrete metric topology.

\begin{figure}[t]
\centering
\resizebox{0.95\textwidth}{!}{%
\begin{tikzpicture}[
    >=Stealth,
    arr/.style={->, thick},
    box/.style={draw, rounded corners=2pt, minimum height=0.6cm, font=\small, align=center, thick},
    databox/.style={box, fill=blue!8, minimum width=1.1cm},
    projbox/.style={box, fill=purple!10, minimum width=2cm},
    netbox/.style={box, fill=green!10, minimum width=2cm},
    predbox/.style={box, fill=yellow!10, minimum width=0.8cm},
    lbl/.style={font=\footnotesize},
]
\node[databox, minimum width=2.2cm] (input) at (0, 0) {Input $x \in \R^d$};
\node[box, fill=orange!10, minimum width=2.2cm, below=0.4 of input] (preproc) {PolyExpand + Scale};
\draw[arr] (input) -- (preproc);
\node[projbox] (p1) at (-3.5, -2.2) {$W_1$ (target-aware)};
\node[projbox] (p2) at (0, -2.2) {$W_2$ (random)};
\node[projbox] (p3) at (3.5, -2.2) {$W_3$ (calibrated)};
\draw[arr] (preproc.south) -- ++(0,-0.25) -| (p1.north);
\draw[arr] (preproc.south) -- (p2.north);
\draw[arr] (preproc.south) -- ++(0,-0.25) -| (p3.north);
\node[box, fill=purple!18, minimum width=0.9cm, font=\footnotesize] (a1) at (-3.5, -3.2) {argsort};
\node[box, fill=purple!18, minimum width=0.9cm, font=\footnotesize] (a2) at (0, -3.2) {argsort};
\node[box, fill=purple!18, minimum width=0.9cm, font=\footnotesize] (a3) at (3.5, -3.2) {argsort};
\draw[arr] (p1) -- (a1);
\draw[arr] (p2) -- (a2);
\draw[arr] (p3) -- (a3);
\node[lbl, below=0.01 of a1, font=\scriptsize] {$\pi^{(1)}$};
\node[lbl, below=0.01 of a2, font=\scriptsize] {$\pi^{(2)}$};
\node[lbl, below=0.01 of a3, font=\scriptsize] {$\pi^{(3)}$};
\node[netbox, minimum height=1cm, below=0.5 of a1] (n1) {ArrowFlow\\Net 1};
\node[netbox, minimum height=1cm, below=0.5 of a2] (n2) {ArrowFlow\\Net 2};
\node[netbox, minimum height=1cm, below=0.5 of a3] (n3) {ArrowFlow\\Net 3};
\draw[arr] (a1) -- ++(0,-0.2) -- (n1);
\draw[arr] (a2) -- ++(0,-0.2) -- (n2);
\draw[arr] (a3) -- ++(0,-0.2) -- (n3);
\node[predbox] (pred1) at (-3.5, -6.0) {$\hat{y}_1{=}$A};
\node[predbox] (pred2) at (0, -6.0) {$\hat{y}_2{=}$B};
\node[predbox] (pred3) at (3.5, -6.0) {$\hat{y}_3{=}$A};
\draw[arr] (n1) -- (pred1);
\draw[arr] (n2) -- (pred2);
\draw[arr] (n3) -- (pred3);
\node[box, fill=red!12, minimum width=4.5cm, minimum height=0.65cm, font=\small\bfseries] (vote) at (0, -7.0) {Majority Vote};
\draw[arr] (pred1.south) -- ++(0,-0.2) -| ([xshift=-1.2cm]vote.north);
\draw[arr] (pred2) -- (vote);
\draw[arr] (pred3.south) -- ++(0,-0.2) -| ([xshift=1.2cm]vote.north);
\node[databox, fill=red!18, below=0.35 of vote, font=\small\bfseries] (final) {$\hat{y} = \text{A}$ \; (2 vs.\ 1)};
\draw[arr] (vote) -- (final);
\draw[decorate, decoration={brace, amplitude=4pt, raise=2pt}, thick] (-4.7, -4.0) -- (-4.7, -3.0);
\node[lbl, rotate=90, anchor=south, font=\scriptsize] at (-5.1, -3.5) {diverse views};
\end{tikzpicture}%
}
\caption{\textbf{Multi-view ensemble.} A single preprocessed input is projected through $K$ different matrices using diverse strategies (target-aware, random, calibrated). Each produces a different permutation---a different ordinal ``view.'' Independent ArrowFlow networks are trained per view; predictions are combined by majority vote. Projection diversity approximates the independence condition of Theorem~\ref{thm:ensemble}.}
\label{fig:ensemble}
\end{figure}
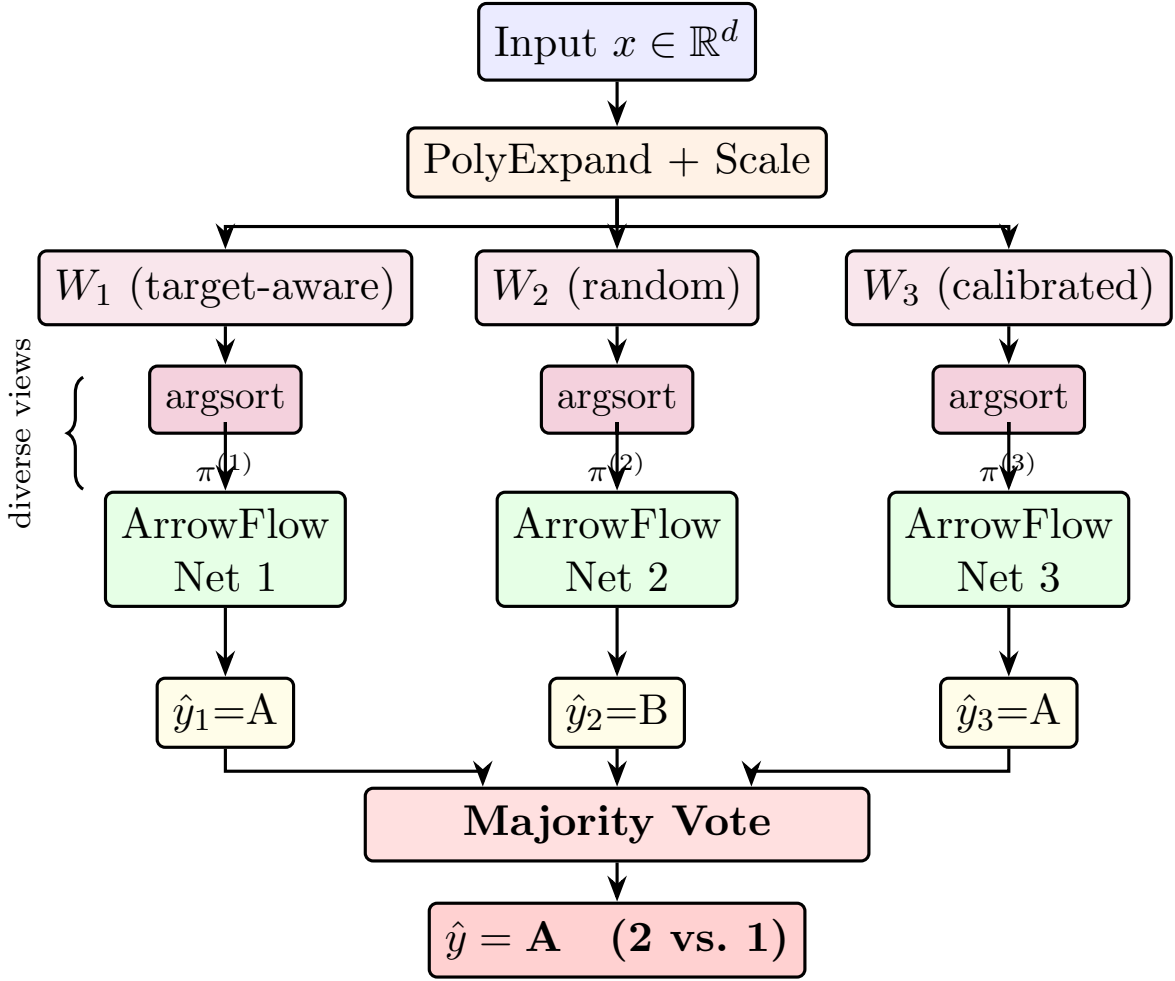

\section{Network Architecture}
\label{sec:architecture}

\subsection{Layer Structure}
Each ArrowFlow network is a layered architecture where each layer is a \textbf{Sort Layer}: a bank of $N$ permutation filters. The layer computes Spearman's footrule distance between the input permutation and each filter, and outputs a ranking of filters by distance. Layers are stacked: the output vocabulary of layer $\ell$ (filter IDs) becomes the input vocabulary of layer $\ell+1$.

The default architecture uses two sort layers: a hidden layer with $N$ filters and an output layer with $C$ filters (one per class). Classification is by nearest-filter: the predicted class is the index of the output-layer filter closest to the final hidden representation.

\paragraph{Layer dimensions and vocabularies.}
The vocabulary changes at every layer. The input is a permutation of the $e$ encoded tokens; layer~1 holds $N_1$ filters, each a permutation of those $e$ tokens, and outputs a ranking of its filters---a permutation in $S_{N_1}$. Layer~2's filters are therefore permutations of the \emph{filter indices} $\{1,\dots,N_1\}$, and so on. The output layer holds $C$ filters (one per class), permutations of $\{1,\dots,N_L\}$, and classification reduces to selecting the nearest class filter.
\begin{center}
\begin{tabular}{lccc}
\toprule
Layer & input vocabulary & \# filters & output \\
\midrule
$1$ & $e$ & $N_1$ & $\pi^{(1)}\in S_{N_1}$ \\
$2$ & $N_1$ & $N_2$ & $\pi^{(2)}\in S_{N_2}$ \\
$\vdots$ & $\vdots$ & $\vdots$ & $\vdots$ \\
output & $N_L$ & $C$ & $\hat{y}=\arg\min_c D^{(L)}_c$ \\
\bottomrule
\end{tabular}
\end{center}
An item absent from an input (e.g.\ under feature masking) is assigned the central position $mV$ with $m=0.5$, so it contributes a moderate, bounded displacement rather than an extreme one; in filter updates its displacement is additionally down-weighted by the deletion-cost weight $w=0.5$.

\subsection{Supervised Training of Sort Layers}
\label{sec:supervised-training}

The forward pass (Section~\ref{sec:ranking-layer}) maps an input permutation $\pi^{(0)} \in S_e$ through layers $\ell = 1, \dots, L$: layer $\ell$ computes $D^{(\ell)}_j = d_F(\pi^{(\ell-1)}, r^{(\ell)}_j)$ for each filter and emits $\pi^{(\ell)} = \argsort(D^{(\ell)})$ (closest filter first). The prediction is the nearest output filter, $\hat{y} = \arg\min_c D^{(L)}_c$. We now specify the complete supervised update, which is non-gradient and built entirely from the accumulation rule of Section~\ref{sec:perm-accum}.

\paragraph{Error-driven gating.}
On example $(x, y)$ with prediction $\hat{y}$, define an update magnitude $g$: if $\hat{y} \neq y$ (a mistake) then $g = 1$; if $\hat{y} = y$ then $g = 1$ with probability $p_c$ and $g = 0$ otherwise, with $p_c = 0.01$. Updates therefore concentrate on misclassified examples---in the spirit of perceptron and LVQ learning---while a small $p_c$ keeps a trickle of reinforcement on already-correct ones.

\paragraph{Output layer: pull the correct-class filter.}
Only the class-$y$ output filter is updated, and only toward the current hidden representation $\pi^{(L-1)}$:
\begin{equation}
\label{eq:output-update}
A^{(L)}_y \leftarrow A^{(L)}_y + \frac{\eta\, g}{N_L}\,\Phi\!\left(\pi^{(L-1)}, r^{(L)}_y\right), \qquad r^{(L)}_y \leftarrow \mathrm{reorder}\!\left(A^{(L)}_y\right),
\end{equation}
where $\Phi$ is the vote matrix (Definition~\ref{def:vote-matrix}) and $\mathrm{reorder}$ is the weighted-average argsort of Eq.~\eqref{eq:accum-reorder}. Every other output filter is left unchanged. At the output layer the update is purely \emph{attractive}: the correct class's prototype moves toward the input, and incorrect-class filters are never pushed away (the repulsion branch described below is used only in hidden layers, never at the output). The learning rate $\eta$ (default $0.1$) sets each example's vote weight $\eta/N_L$ relative to the unit-weight identity prior in the accumulator; it scales votes, not a gradient.

\paragraph{Hidden layers: motion propagation.}
The signed displacement field $\mu^{(L)}$ produced by the output update---how far each element of $\pi^{(L-1)}$ must move to reach $r^{(L)}_y$---is the supervised signal handed to the layer below. For $\ell = L-1$ down to $1$, given the incoming target motion $\mu^{(\ell+1)}$:
\begin{enumerate}[leftmargin=2em,topsep=2pt]
\item \textbf{Top-$\rho$ selection.} Order the layer's filters by $|\mu^{(\ell+1)}|$ and update only the top $\rho N_\ell$ of them ($\rho = 0.5$); the small-motion remainder is left untouched---a sparsity that both regularizes and halves the update cost.
\item \textbf{Signed accumulation.} Each selected filter $r^{(\ell)}_j$ accumulates votes with \emph{nonnegative} weight $|\mu^{(\ell+1)}_j|$ toward a target $\tilde{\pi}$ whose choice is set by the \emph{sign} of the propagated motion---the layer input $\pi^{(\ell-1)}$ when $\mu^{(\ell+1)}_j \ge 0$ (attractive), or its reversal $\mathrm{rev}(\pi^{(\ell-1)})$ when $\mu^{(\ell+1)}_j < 0$ (repulsive):
\begin{equation}
\label{eq:hidden-update}
A^{(\ell)}_j \leftarrow A^{(\ell)}_j + \frac{2\eta}{N_\ell}\,\bigl|\mu^{(\ell+1)}_j\bigr|\,\Phi\!\left(\tilde{\pi}, r^{(\ell)}_j\right), \qquad
r^{(\ell)}_j \leftarrow \mathrm{reorder}(A^{(\ell)}_j),
\end{equation}
with $\tilde{\pi} = \pi^{(\ell-1)}$ for $\mu^{(\ell+1)}_j \ge 0$ and $\tilde{\pi} = \mathrm{rev}(\pi^{(\ell-1)})$ otherwise, as above.
Because the weights $|\mu^{(\ell+1)}_j|$ are nonnegative, each $A^{(\ell)}_j$ remains a vote-count matrix, so the Borda/mean-position reading of Proposition~\ref{prop:convergence} still applies (that proposition analyzes the accumulator mechanics; it does not model the sign-dependent choice of $\tilde{\pi}$, which is what makes the hidden update a signed, attract-or-repel rule).
\item \textbf{Cross-filter aggregation.} The motion passed further down is the \emph{mean} over the layer's filters of their per-item displacements, renormalized to peak magnitude $c\,N_\ell$ ($c = 0.125$) and sorted by amplitude:
\[
\bar{\mu} = \frac{1}{N_\ell}\sum_j \mu^{(\ell)}_{j}, \qquad
\mu^{(\ell)} = \mathrm{sort}_{|\cdot|}\!\left(c\,N_\ell\,\frac{\bar{\mu}}{\max_i |\bar{\mu}_i|}\right).
\]
\end{enumerate}
This mean aggregation is ArrowFlow's substitute for backpropagating through the non-differentiable argsort: in place of a chain rule, each layer summarizes the displacement demanded of all its filters into a single ordinal target for the layer below.

\paragraph{Frozen output layer.}
The strongest empirical setting (Table~\ref{tab:llu}) freezes the output layer. Mechanistically, the class-$y$ filter's motion $\mu^{(L)}$ is still computed and propagated to drive the hidden layers, but the output accumulator is \emph{discarded} rather than applied, so the $C$ output filters remain their random initial permutations---fixed, one-per-class prototypes. The hidden layers then learn to map each class's inputs close to its own fixed prototype. Freezing supplies a stable supervised target and prevents the prototypes from drifting toward---and collapsing onto---the moving hidden representations; it is the ordinal analogue of fixing a randomly-initialized classifier head. The benefit is consistent across datasets and the five random-seed simulations of Table~\ref{tab:llu}.

\begin{algorithm}[t]
\caption{Supervised training step (one example, sort-only network)}
\label{alg:supervised}
\begin{algorithmic}[1]
\Require example $(x,y)$; filters $\{r^{(\ell)}_j\}$, accumulators $\{A^{(\ell)}_j\}$; rate $\eta$, fraction $\rho$, scale $c$, correct-update prob.\ $p_c$
\State $\pi^{(0)} \gets \mathrm{encode}(x)$
\For{$\ell = 1$ to $L$} \Comment{forward}
  \State $D^{(\ell)}_j \gets d_F(\pi^{(\ell-1)}, r^{(\ell)}_j)\ \forall j$; \quad $\pi^{(\ell)} \gets \argsort(D^{(\ell)})$
\EndFor
\State $\hat{y} \gets \arg\min_c D^{(L)}_c$
\State $g \gets 1$ if $\hat{y} \neq y$; \ else $g \gets 1$ w.p.\ $p_c$, else $0$ \Comment{error-driven gating}
\State $\mu^{(L)} \gets$ displacement of $\pi^{(L-1)}$ toward $r^{(L)}_y$ \Comment{computed even if output frozen}
\If{output layer not frozen}
  \State $A^{(L)}_y \mathrel{+}= (\eta g / N_L)\,\Phi(\pi^{(L-1)}, r^{(L)}_y)$; \quad $r^{(L)}_y \gets \mathrm{reorder}(A^{(L)}_y)$
\EndIf
\For{$\ell = L-1$ down to $1$} \Comment{motion propagation}
  \State $S \gets$ top $\rho N_\ell$ filters by $|\mu^{(\ell+1)}|$
  \ForAll{$j \in S$}
     \State $\tilde{\pi} \gets \pi^{(\ell-1)}$ if $\mu^{(\ell+1)}_j \ge 0$ else $\mathrm{rev}(\pi^{(\ell-1)})$ \Comment{sign $\Rightarrow$ attract/repel}
     \State $A^{(\ell)}_j \mathrel{+}= (2\eta/N_\ell)\,|\mu^{(\ell+1)}_j|\,\Phi(\tilde{\pi}, r^{(\ell)}_j)$; \quad $r^{(\ell)}_j \gets \mathrm{reorder}(A^{(\ell)}_j)$
  \EndFor
  \State $\bar{\mu} \gets \tfrac{1}{N_\ell}\sum_j \mu^{(\ell)}_j$; \quad $\mu^{(\ell)} \gets \mathrm{sort}_{|\cdot|}\!\big(c N_\ell\,\bar{\mu}/\max_i|\bar{\mu}_i|\big)$
\EndFor
\end{algorithmic}
\end{algorithm}

\paragraph{Remaining details.}
Filters are initialized as uniform random permutations of their vocabulary; training presents data over $200$--$500$ iterations (the forward pass is batched and updates accumulate per example). Ties in $\argsort$---rare for integer footrule distances---are broken by index, making the forward pass deterministic given the filters. All experiments use the footrule norm $q = 1$; the $q \in \{0, 2\}$ variants of Eq.~\eqref{eq:distance} are available but unused. The default constants are collected for reference: $\eta = 0.1$ (rate), $\rho = 0.5$ (fraction of filters updated per layer), $c = 0.125$ (motion renormalization), $p_c = 0.01$ (correct-example update probability), $m = 0.5$ (missing-item position), $w = 0.5$ (deletion-cost weight).

\section{Theoretical Analysis}
\label{sec:theory}

This section establishes formal properties of ArrowFlow's components: the stability of ordinal encoding under perturbation, the information-theoretic capacity of argsort, the amplification of noise by polynomial expansion, the convergence of the learning rule, and the ensemble error rate. Several results are validated empirically in the experiments that follow (Sections~\ref{sec:experiments}--\ref{sec:utility}); we give forward pointers where a theorem predicts a specific experimental outcome.

\subsection{Metric Properties of the Permutation Space}

We first state known results that ground ArrowFlow's choice of distance metric.

\begin{lemma}[Footrule--Kendall Equivalence {\citep{diaconis1977footrule}}]
\label{lem:fk-equiv}
For any $\pi, \sigma \in S_V$,
\[
d_K(\pi, \sigma) \;\leq\; d_F(\pi, \sigma) \;\leq\; 2\, d_K(\pi, \sigma),
\]
where $d_K$ counts pairwise disagreements (Kendall tau) and $d_F = \sum_{i=1}^V |\pi(i) - \sigma(i)|$ is Spearman's footrule. The diameter of $(S_V, d_F)$ is $\lfloor V^2/2 \rfloor$, attained by the identity and reverse permutations.
\end{lemma}

This equivalence means that footrule and Kendall tau induce the same topology on $S_V$. ArrowFlow uses footrule because it decomposes as a sum of independent per-item displacements, enabling the accumulation-based learning rule of Section~\ref{sec:perm-accum}.

\subsection{Argsort Stability Under Perturbation}

The ordinal encoding $\pi = \argsort(x)$ is piecewise constant: invariant within each \emph{permutation cone} but discontinuous at boundaries where two coordinates are equal. The following theorem quantifies the stability region.

\begin{theorem}[Argsort Stability]
\label{thm:stability}
Let $x \in \R^d$ with all components distinct. Define the \emph{minimum gap}
\[
\delta_{\min}(x) = \min_{i \neq j} |x_i - x_j|.
\]
If $\varepsilon \in \R^d$ satisfies $\|\varepsilon\|_\infty < \delta_{\min}(x)/2$, then $\argsort(x + \varepsilon) = \argsort(x)$.
\end{theorem}

\begin{proof}
For any pair $i, j$ with $x_i > x_j$, we require $(x_i + \varepsilon_i) > (x_j + \varepsilon_j)$, equivalently $\varepsilon_j - \varepsilon_i < x_i - x_j$. Since $|\varepsilon_j - \varepsilon_i| \leq 2\|\varepsilon\|_\infty < \delta_{\min}(x) \leq x_i - x_j$, the strict inequality holds. As this applies to all ordered pairs, the complete ordering is preserved.
\end{proof}

\begin{corollary}[Gaussian Stability Bound]
\label{cor:gaussian}
If $\varepsilon \sim \mathcal{N}(0, \sigma^2 I_d)$, then
\[
\Pr\big[\argsort(x + \varepsilon) \neq \argsort(x)\big] \;\leq\; \binom{d}{2} \exp\!\left(-\frac{\delta_{\min}(x)^2}{4\sigma^2}\right).
\]
\end{corollary}

\begin{proof}
For each ordered pair $(i,j)$ with $x_i > x_j$, the reversal event is $\{\varepsilon_j - \varepsilon_i \geq x_i - x_j\}$. Since $\varepsilon_j - \varepsilon_i \sim \mathcal{N}(0, 2\sigma^2)$, the Chernoff bound gives $\Pr[\varepsilon_j - \varepsilon_i \geq t] \leq \exp(-t^2/(4\sigma^2))$ for $t > 0$. Setting $t = x_i - x_j \geq \delta_{\min}$ and applying the union bound over all $\binom{d}{2}$ pairs yields the result. (A tighter bound with an additional factor of $1/2$ follows from the complementary error function.)
\end{proof}

Theorem~\ref{thm:stability} and Corollary~\ref{cor:gaussian} are stated for the vector that is \emph{actually sorted}. In the ArrowFlow pipeline that vector is the projected score $z = \phi(x)W \in \R^e$, not the raw input $x$, so the bounds should be read at the score level. The following corollary translates raw-input noise through the (degree-1) projection.

\begin{corollary}[Projected-Noise Stability]
\label{cor:proj-noise}
Let $z = xW \in \R^e$ (degree-1 features) and let $\varepsilon_x \in \R^d$ be raw-input noise, so the score perturbation is $\varepsilon_z = \varepsilon_x W$. Then $\argsort(z + \varepsilon_z) = \argsort(z)$ whenever $\|\varepsilon_x W\|_\infty < \delta_{\min}(z)/2$. If $\varepsilon_x \sim \mathcal{N}(0, \sigma^2 I_d)$, then for each coordinate pair $(i,j)$ the difference $\varepsilon_{z,j} - \varepsilon_{z,i}$ is Gaussian with variance $\sigma^2\|W_{\cdot i} - W_{\cdot j}\|^2$---\emph{correlated} across pairs through shared columns of $W$---and
\[
\Pr\big[\argsort(z + \varepsilon_z) \neq \argsort(z)\big] \;\leq\; \sum_{i<j} \exp\!\left(-\frac{\delta_{\min}(z)^2}{2\,\sigma^2\,\|W_{\cdot i} - W_{\cdot j}\|^2}\right).
\]
\end{corollary}

\begin{proof}
The deterministic claim is Theorem~\ref{thm:stability} applied to $z$. For the Gaussian claim, $\varepsilon_{z,k} = \sum_m \varepsilon_{x,m} W_{mk}$, so $\varepsilon_{z,j} - \varepsilon_{z,i} = \sum_m \varepsilon_{x,m}(W_{mj} - W_{mi}) \sim \mathcal{N}(0, \sigma^2\|W_{\cdot j} - W_{\cdot i}\|^2)$. The reversal of pair $(i,j)$ requires this difference to exceed $z_i - z_j \geq \delta_{\min}(z)$; the Chernoff bound and a union bound over the $\binom{e}{2}$ pairs give the result.
\end{proof}

This makes the dependence on the projection explicit: isotropic raw noise becomes anisotropic, correlated score noise whose effective scale per pair is set by the column-difference norms $\|W_{\cdot i} - W_{\cdot j}\|$.

\begin{remark}
Read at the level of the sorted score vector, this analysis directly explains the empirical findings of Section~\ref{sec:noise}: ArrowFlow's noise robustness is governed by $\delta_{\min}/\sigma$, the ratio of the minimum feature gap to the noise scale. On high-dimensional data like Digits ($d = 64$), the pixel intensities create many distinct values with reasonably large gaps, making the encoding robust. The bound also clarifies why the advantage vanishes with polynomial expansion: expanded features have many more near-equal values, shrinking $\delta_{\min}$.
\end{remark}

\subsection{Information Capacity of Ordinal Encoding}

\begin{theorem}[Ordinal Information Capacity]
\label{thm:capacity}
The argsort encoding $\argsort\colon \R^d \to S_d$ partitions the generic points of $\R^d$ (those with distinct coordinates) into exactly $d!$ equivalence classes (the orientation---ascending or descending---only relabels the classes and does not change their number). Writing $\pi=\orderdesc(x)$ for the decreasing order, each class is an open convex cone
\[
C_\pi = \big\{x \in \R^d : x_{\pi(1)} > x_{\pi(2)} > \cdots > x_{\pi(d)}\big\},
\]
called a \emph{permutation cone} (Weyl chamber of the symmetric group). The information content of the encoding is exactly $\log_2(d!)$ bits, satisfying
\begin{equation}
\label{eq:capacity}
\log_2(d!) = d\log_2 d - d\log_2 e + \tfrac{1}{2}\log_2(2\pi d) + O(1/d).
\end{equation}
\end{theorem}

\begin{proof}
Two points $x, y \in \R^d$ satisfy $\argsort(x) = \argsort(y)$ if and only if for all $i \neq j$: $x_i > x_j \Leftrightarrow y_i > y_j$. This defines $d!$ open cones, one per permutation $\pi \in S_d$, each nonempty and convex (defined by strict linear inequalities). The boundary $\{x : x_i = x_j \text{ for some } i \neq j\}$ has Lebesgue measure zero. Since argsort maps each cone to a distinct permutation, the encoding has exactly $\log_2(d!)$ bits of information. The asymptotic expansion follows from Stirling's formula $d! \approx \sqrt{2\pi d}\,(d/e)^d$.
\end{proof}

For $d = 64$ (Digits), the capacity is $\log_2(64!) \approx 296$ bits---substantial, but finite and far below the infinite capacity of real-valued representations. A subtlety matters for the full pipeline: ArrowFlow does not argsort the expanded features directly but a \emph{projected} score vector $z = \phi(x)W \in \R^e$ (Section~\ref{sec:encoding}). The relevant ordinal alphabet is therefore $S_e$, with maximum capacity $\log_2(e!)$ controlled by the embedding dimension $e$---\emph{not} $\log_2(d'!)$. Polynomial expansion from $d$ to $d' = \binom{d+k}{k}-1$ features does \emph{not} enlarge this alphabet when $e$ is fixed; instead it changes the geometry of the map $x \mapsto z$, and hence which permutation cones the data reach. For Digits with degree~1 and $e = d = 64$ the two coincide, which is why $\log_2(64!)$ applies there.

\begin{figure}[t]
\centering
\begin{tikzpicture}[scale=1.0]
\draw[->, thick, gray] (0,0) -- (3.5,0) node[right, font=\scriptsize] {$x_1$};
\draw[->, thick, gray] (0,0) -- (-2.2,-1.3) node[left, font=\scriptsize] {$x_2$};
\draw[->, thick, gray] (0,0) -- (0,3.8) node[above, font=\scriptsize] {$x_3$};
\draw[dashed, thick, blue!50] (-2, -1.2) -- (3.2, 1.9) node[right, font=\tiny, blue!70] {$x_1{=}x_2$};
\draw[dashed, thick, red!50] (-1.5, -2) -- (1, 3.5) node[above, font=\tiny, red!70] {$x_1{=}x_3$};
\draw[dashed, thick, green!60!black] (-2.5, 0.5) -- (3, 0.5) node[right, font=\tiny, green!60!black] {$x_2{=}x_3$};
\node[font=\scriptsize, fill=white, inner sep=2pt, rounded corners=2pt] at (2.3, 2.5) {$x_1{>}x_3{>}x_2$};
\node[font=\scriptsize, fill=white, inner sep=2pt, rounded corners=2pt] at (2.5, -0.3) {$x_1{>}x_2{>}x_3$};
\node[font=\scriptsize, fill=white, inner sep=2pt, rounded corners=2pt] at (-0.7, 3.0) {$x_3{>}x_1{>}x_2$};
\node[font=\scriptsize, fill=white, inner sep=2pt, rounded corners=2pt] at (0.6, 0.7) {$x_2{>}x_1{>}x_3$};
\node[font=\scriptsize, fill=white, inner sep=2pt, rounded corners=2pt] at (-1.5, 1.5) {$x_3{>}x_2{>}x_1$};
\node[font=\scriptsize, fill=white, inner sep=2pt, rounded corners=2pt] at (-1.5, -0.5) {$x_2{>}x_3{>}x_1$};
\fill[black] (1.8, 1.1) circle (2pt);
\node[font=\tiny, right] at (1.9, 1.05) {$(3, 1, 2)$};
\draw[->, thick, black!60] (1.8, 1.1) -- (2.1, 2.1);
\node[font=\tiny, right, black!60] at (2.15, 1.6) {$\argsort = [1,3,2]$};
\draw[thick, dotted, black!40] (1.8, 1.1) circle (0.35);
\node[font=\tiny, black!50, anchor=west] at (2.2, 0.75) {$\delta_{\min}/2$};
\end{tikzpicture}
\caption{\textbf{Permutation cones.} The argsort encoding partitions $\R^d$ into $d!$ convex cones (Weyl chambers), separated by hyperplanes $\{x_i = x_j\}$. Shown for $d=3$: six orderings of three coordinates. All points within a cone map to the same permutation. The dotted circle illustrates the stability radius from Theorem~\ref{thm:stability}: perturbations smaller than $\delta_{\min}(x)/2$ cannot cross a boundary.}
\label{fig:cones}
\end{figure}
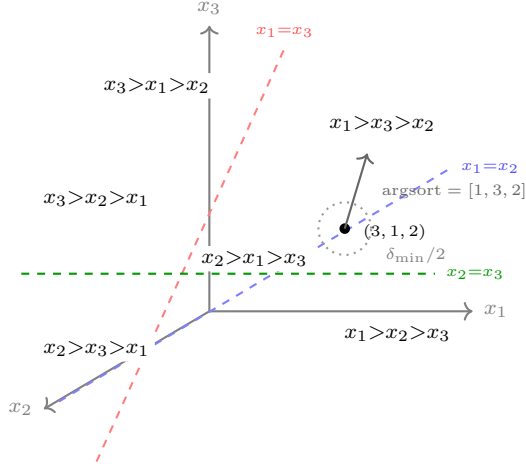

\subsection{Polynomial Noise Amplification}

The following result formalizes the fundamental tension between the information gain of polynomial expansion and its noise cost, providing a theoretical foundation for the empirical trade-off of Section~\ref{sec:tradeoff}.

\begin{proposition}[Polynomial Noise Amplification]
\label{prop:poly-noise}
Let $x \in \R^d$ with $\|x\|_\infty \leq B$, and let $\varepsilon \sim \mathcal{N}(0, \sigma^2 I_d)$. For a degree-$k$ monomial feature $f(x) = x_{j_1} x_{j_2} \cdots x_{j_k}$ with \emph{distinct} indices $j_1, \ldots, j_k$, the perturbation satisfies
\[
\mathrm{Var}\big[f(x+\varepsilon) - f(x)\big] = \sigma^2 \|\nabla f(x)\|^2 + O(\sigma^4) \;=\; \sigma^2 \sum_{s=1}^k \prod_{t \neq s} x_{j_t}^2 \;+\; O(\sigma^4),
\]
which is at most $O(k\, \sigma^2 B^{2(k-1)})$.
Consequently, the effective noise standard deviation per feature in the expanded space grows as $O(\sqrt{k}\, \sigma B^{k-1})$, weakening the stability bound of Theorem~\ref{thm:stability} by a factor of $B^{k-1}$ relative to degree-1 features. This is a \emph{local}, first-order (delta-method) sensitivity statement for a single distinct-index monomial; bridging it to a guarantee on argsort stability over all $\binom{d+k}{k}-1$ expanded features requires an additional union bound (Remark below), and the constant changes for repeated-index monomials such as $x_i^k$. We therefore use it to \emph{help explain}, not to fully prove, the empirical degradation of Section~\ref{sec:tradeoff}.
\end{proposition}

\begin{proof}
Expanding $f(x+\varepsilon) = \prod_{s=1}^k (x_{j_s} + \varepsilon_{j_s})$ and subtracting $f(x)$, the first-order term is $\nabla f(x) \cdot \varepsilon = \sum_{s=1}^k \varepsilon_{j_s} \prod_{t \neq s} x_{j_t}$. Since the indices are distinct, the $\varepsilon_{j_s}$ are independent, so the variance is $\sigma^2 \sum_{s=1}^k \prod_{t \neq s} x_{j_t}^2$. Higher-order cross terms contribute $O(\sigma^4)$. Each product satisfies $\prod_{t \neq s} x_{j_t}^2 \leq B^{2(k-1)}$, giving the stated bound.
\end{proof}

\begin{remark}
For monomials with repeated indices (e.g., $x_i^k$), the variance is $\sigma^2 (\partial f/\partial x_i)^2 = k^2 \sigma^2 x_i^{2(k-1)}$, which is larger by a factor of $k$ than the distinct-index case due to the multiplicity in the partial derivative.
\end{remark}

\begin{remark}
This proposition, combined with Theorem~\ref{thm:stability}, yields a qualitative explanation of the noise--capacity trade-off. At degree $k=1$, argsort is stable when $\sigma < \delta_{\min}/2$; at degree $k=2$, the per-feature noise standard deviation grows by a factor of $B$, so $\delta_{\min}$ in the expanded space must be correspondingly larger to maintain stability. For standardized data ($B \approx 2\text{--}3$), this explains the 2--3$\times$ faster degradation observed empirically with polynomial expansion. A fully rigorous bridge from per-feature variance to the $\ell_\infty$ noise bound required by Theorem~\ref{thm:stability} would additionally require a union bound over all $\binom{d+k}{k}-1$ expanded features, introducing a factor of $\sqrt{\log M}$ where $M$ is the number of features.
\end{remark}

\subsection{Convergence of Permutation-Matrix Accumulation}

The accumulation rule (Eqs.~\ref{eq:accum-update}--\ref{eq:accum-reorder}) can be analyzed as a statistical estimator of the mean-position (Borda) ranking.

\begin{proposition}[Consistency of the Mean-Position Accumulator]
\label{prop:convergence}
Let $\pi_1, \ldots, \pi_T$ be i.i.d.\ samples from a distribution $P$ over $S_V$ with distinct expected positions: $\E[\pi(u)] \neq \E[\pi(v)]$ for all $u \neq v$. Define the \emph{mean-position ranking} $\pi_{\mathrm{mean}} = \argsort\big(\E[\pi(1)], \ldots, \E[\pi(V)]\big)$. Then the filter ordering $r_T$ produced by the accumulation rule satisfies $r_T \to \pi_{\mathrm{mean}}$ almost surely as $T \to \infty$. This is a Borda/mean-rank aggregate; it is \emph{not} in general the footrule median (Theorem~\ref{thm:mle-filter}).
\end{proposition}

\begin{proof}
The implementation initializes $A$ to the identity (a prior encoding the current filter ordering). After $T$ updates, $A[v, k] = \mathbf{1}[v{=}k] + \sum_{t=1}^T \mathbf{1}[\pi_t \text{ places } v \text{ at position } k]$. The weighted-average estimate is
\[
\hat{\imath}_T(v) = \frac{\sum_k k \cdot A[v,k]}{\sum_k A[v,k]} = \frac{v + \sum_{t=1}^T \pi_t(v)}{1 + T},
\]
which is a convex combination of the prior position $v$ and the sample mean $(1/T)\sum_{t=1}^T \pi_t(v)$. As $T \to \infty$, the prior weight $1/(1+T) \to 0$, so $\hat{\imath}_T(v) \to \E[\pi(v)]$ almost surely by the strong law of large numbers. Since argsort is locally constant on the open set $\{y \in \R^V : y_u \neq y_v\ \forall\, u \neq v\}$, and the expected positions are distinct by assumption, eventually $\hat{\imath}_T$ lies in a neighborhood where argsort is constant, giving $r_T = \argsort(\hat{\imath}_T(1), \ldots, \hat{\imath}_T(V)) \to \pi_{\mathrm{mean}}$ almost surely.
\end{proof}

\begin{remark}
For the Mallows model $P(\pi \mid \pi^*, \lambda) \propto \exp(-\lambda \cdot d_K(\pi, \pi^*))$ \citep{mallows1957nonnull} (originally defined with Kendall tau $d_K$; footrule variants also apply), the mean-position ranking coincides with the modal permutation, $\pi_{\mathrm{mean}} = \pi^*$, when $\lambda$ is sufficiently large. By the central limit theorem, $\hat{\imath}_T(v)$ has standard deviation $\sigma_v / \sqrt{T}$ where $\sigma_v^2 = \mathrm{Var}[\pi(v)] \leq V^2/4$. A union bound over all $V$ items gives that $\max_v |\hat{\imath}_T(v) - \E[\pi(v)]| < \gamma/2$ with high probability when $T = \Omega(V^2 \log V / \gamma^2)$, where $\gamma = \min_{u \neq v}|\E[\pi(u)] - \E[\pi(v)]|$ is the minimum gap between expected positions.
\end{remark}

\subsection{Multi-View Ensemble Error Bound}

\begin{theorem}[Ensemble Error Bound]
\label{thm:ensemble}
Let $K$ (odd) views produce independent binary predictions, each with error probability $p < 1/2$. The majority-vote error satisfies
\[
P_{\mathrm{err}}^{(\mathrm{MV})} \;\leq\; \exp\!\Big(-2K\big(\tfrac{1}{2} - p\big)^2\Big),
\]
which decreases exponentially in $K$.
\end{theorem}

\begin{proof}
Let $X_k \in \{0,1\}$ indicate a correct prediction by view $k$, with $\E[X_k] = 1-p > 1/2$. The majority vote errs when $\sum_k X_k \leq K/2$, i.e., $\bar{X} \leq 1/2$. Since $\E[\bar{X}] = 1-p$, Hoeffding's inequality gives $\Pr[\bar{X} \leq 1/2] \leq \exp(-2K(1/2 - p)^2)$.
\end{proof}

\begin{remark}
The independence assumption is critical and is the primary reason ArrowFlow uses diverse projection strategies. If all views used the same projection, their errors would be perfectly correlated and the ensemble would provide no benefit. Under full independence, $K = 7$ views with per-view error $p = 0.20$ would yield a 6$\times$ error reduction (exact binomial). Empirically, ArrowFlow's 7 views reduce error by $1.3$--$3.3\times$ (Table~\ref{tab:progression}), which is less than the independent-view prediction, indicating that projection diversity achieves only partial decorrelation. The gap between theoretical and observed reduction quantifies the residual correlation among views---an important factor when choosing the number of views. This bound is therefore best read as an \emph{idealized} (independent, binary) benchmark: under pairwise error correlation $\rho > 0$ the effective number of independent voters is roughly $K/(1 + (K-1)\rho)$, so the exponential rate degrades accordingly, and the tasks here are multiclass rather than binary. A direct measurement of inter-view error correlation would turn this from generic theory into evidence about the actual architecture; we leave that to future work.
\end{remark}

\subsection{Social Choice Foundations: From Arrow's Theorem to Learning Guarantees}
\label{sec:social-choice-theory}

Section~\ref{sec:arrow} presented the Arrow connection as a design principle.  Here we make it precise, which requires one clarification at the outset.  The \emph{forward} pass of a ranking layer orders filters by their \emph{independent} distances $\{D_j\}$, so the relative order of filters $a$ and $b$ depends only on $D_a$ and $D_b$; the forward layer therefore \emph{satisfies} pairwise independence of irrelevant alternatives, and its nonlinearity originates in the piecewise-constant cone geometry of $\argsort$ (Theorem~\ref{thm:capacity}), not in any IIA violation.  The genuine social-choice content lives in the \emph{learning rule}: the accumulation update (Eqs.~\ref{eq:accum-update}--\ref{eq:accum-reorder}) aggregates the input rankings a filter accumulates into a consensus ordering, and \emph{this} aggregation is the object to which Arrow's and Gibbard--Satterthwaite's theorems apply.  We state three results that characterize \emph{what} the aggregation estimates (Theorem~\ref{thm:mle-filter} and Proposition~\ref{prop:exp-consistency}), \emph{how far} a layer's output can be moved by an adversarial input perturbation (Theorem~\ref{thm:manipulability}), and \emph{that} the aggregation provably violates IIA in the exact sense of Arrow's theorem (Proposition~\ref{prop:iia-index}).

\subsubsection{Filter Learning as Maximum Likelihood Estimation}

The \emph{footrule-Mallows model} \citep{mallows1957nonnull} places an exponential-family distribution on $S_V$:
\begin{equation}
\label{eq:mallows}
P(\pi \mid \pi^*, \lambda) \;=\; \frac{1}{Z(\lambda)}\,\exp\!\bigl(-\lambda\, d_F(\pi, \pi^*)\bigr),
\end{equation}
where $d_F$ is Spearman's footrule, $\pi^*$ is the modal (central) permutation, $\lambda > 0$ is a concentration parameter, and $Z(\lambda) = \sum_{\sigma \in S_V} \exp(-\lambda\, d_F(\sigma, \pi^*))$ is the normalizing constant.  The model is a natural generalization of the Gaussian: $\lambda$ controls dispersion and the mode $\pi^*$ is the permutation analogue of the mean.

\begin{theorem}[Filter Learning as Consistent Rank Aggregation]
\label{thm:mle-filter}
Let $\pi_1, \dots, \pi_T \in S_V$ be i.i.d.\ samples from the footrule-Mallows model $P(\cdot \mid \pi^*, \lambda)$ with fixed $\lambda > 0$.
\begin{enumerate}[label=(\roman*)]
\item The maximum-likelihood estimator of the central permutation is the \emph{footrule median}
$\hat{\pi}^*_{\mathrm{MLE}} = \arg\min_{\sigma \in S_V} \sum_{t=1}^{T} d_F(\pi_t, \sigma)$.
\item ArrowFlow's accumulation rule (Eqs.~\ref{eq:accum-update}--\ref{eq:accum-reorder}) instead computes the \emph{positional} (Borda) aggregate
$r_T = \argsort\!\bigl(\tfrac1T\sum_t \pi_t(1), \dots, \tfrac1T\sum_t \pi_t(V)\bigr)$.
\item Both estimators are consistent for the same target: $\hat{\pi}^*_{\mathrm{MLE}} \to \pi^*$ and $r_T \to \pi^*$ almost surely as $T \to \infty$.
\end{enumerate}
\end{theorem}

\begin{proof}
\emph{(i)} The log-likelihood is $\ell(\sigma) = -\lambda \sum_t d_F(\pi_t, \sigma) - T\log Z(\lambda)$; since $\lambda > 0$ and $Z(\lambda)$ is independent of $\sigma$, maximizing $\ell$ is equivalent to minimizing $\sum_t d_F(\pi_t, \sigma)$, the footrule median (the footrule analogue of Kemeny aggregation \citep{dwork2001rank}).
\emph{(ii)} is exactly Eq.~\eqref{eq:accum-reorder} once the identity prior is amortized (Proposition~\ref{prop:convergence}).
\emph{(iii)} By the strong law, $\tfrac1T\sum_t \pi_t(v) \to \E[\pi(v)]$ a.s.\ for each item $v$. We assume the standard regularity condition that the expected positions are ordered consistently with the mode, $\argsort(\E[\pi(\cdot)]) = \pi^*$; this holds for the footrule-Mallows model at sufficiently large concentration $\lambda$ (intuitively, the symmetric unimodal law concentrates each item's position around its modal rank, though we do not prove the monotonicity for all $\lambda$ here). Under this condition $\pi_{\mathrm{mean}} = \argsort(\E[\pi(\cdot)]) = \pi^*$, and Proposition~\ref{prop:convergence} gives $r_T \to \pi_{\mathrm{mean}} = \pi^*$. The same symmetry makes $\pi^*$ the unique \emph{population} footrule median, so the empirical footrule-risk minimizer satisfies $\hat{\pi}^*_{\mathrm{MLE}} \to \pi^*$ by standard $M$-estimator consistency on the finite set $S_V$. Both estimators thus share the limit $\pi^*$, though they generally differ at finite $T$ (Borda mean-position vs.\ footrule median).
\end{proof}

\begin{remark}
The distinction in parts (i)--(ii) matters: the accumulation rule is \emph{not} the footrule-median MLE at finite $T$---it is the cheaper $O(VT)$ positional statistic---but it is consistent for the same modal permutation. This is the honest version of the folklore ``Borda approximates Kemeny'': they agree in the limit, not in general. A minimal example makes the finite-sample gap concrete. For the profile $ABC,\,ABC,\,BCA,\,CBA$ over three items, the mean positions are $(2.00,\,1.75,\,2.25)$, so the accumulator returns $BAC$; the footrule median, however, is $ABC$ (footrule objective $8$ versus $10$ for $BAC$). The two disagree even though the mean positions are distinct, so this is not a tie-breaking artifact.
\end{remark}

\begin{proposition}[Exponential Consistency]
\label{prop:exp-consistency}
Under the footrule-Mallows model with fixed $\lambda > 0$, the MLE identifies the true mode with error probability decaying exponentially in $T$:
\[
\Pr\bigl[\hat{\pi}^*_{\mathrm{MLE}} \neq \pi^*\bigr] \;\le\; (V!-1)\,\exp\!\Bigl(-\frac{T\,\Delta_\lambda^2}{2\,\mathrm{diam}^2}\Bigr),
\]
where $\mathrm{diam} = \lfloor V^2/2\rfloor$ is the footrule diameter (Lemma~\ref{lem:fk-equiv}) and
$\Delta_\lambda = \min_{\sigma \neq \pi^*}\bigl(\E\,d_F(\pi,\sigma) - \E\,d_F(\pi,\pi^*)\bigr) > 0$ is the population footrule-risk gap.
\end{proposition}

\begin{proof}
The MLE prefers some $\sigma \neq \pi^*$ only if $\tfrac1T\sum_t\bigl[d_F(\pi_t,\sigma) - d_F(\pi_t,\pi^*)\bigr] \le 0$. Each summand lies in $[-\mathrm{diam}, \mathrm{diam}]$ (range $2\,\mathrm{diam}$) and has mean $\ge \Delta_\lambda > 0$ (positivity holds because $\pi^*$ is the unique population minimizer of the footrule risk under the symmetric unimodal Mallows law). Hoeffding's inequality for a sum of $T$ i.i.d.\ variables of range $2\,\mathrm{diam}$ bounds this event by $\exp\bigl(-2T\Delta_\lambda^2/(2\,\mathrm{diam})^2\bigr) = \exp\bigl(-T\Delta_\lambda^2/(2\,\mathrm{diam}^2)\bigr)$; a union bound over the $V!-1$ competing permutations gives the claim.
\end{proof}

\begin{remark}
A Cram\'er--Rao bound does \emph{not} apply here: $\pi^*$ is a discrete parameter, so there is no Fisher information in the classical sense, and earlier drafts that invoked it were mistaken. Exponential consistency is the correct optimality notion---the misidentification probability vanishes geometrically at a rate governed by the footrule-risk gap $\Delta_\lambda$, which grows with the concentration $\lambda$. This connects ArrowFlow's learning rule to a 60-year tradition in rank aggregation: the Kemeny rule is the MLE for the Kendall-Mallows model \citep{conitzer2005voting}, and Theorem~\ref{thm:mle-filter} is the footrule analogue, with the accumulation rule reaching the same limit through cheaper positional aggregation.
\end{remark}

\subsubsection{Adversarial Perturbation Bounds}

The Gibbard--Satterthwaite theorem \citep{gibbard1973manipulation,satterthwaite1975strategy} states that every non-dictatorial, surjective social choice function over three or more alternatives is \emph{manipulable}.  For ArrowFlow this is a motivating analogy rather than a proof device: it suggests that any non-trivial layer must admit input perturbations that change its output.  We make this quantitative with an elementary two-sided bound that follows directly from the metric structure of the footrule distance, with no appeal to social-choice axioms.

\begin{theorem}[Adversarial Perturbation Bounds]
\label{thm:manipulability}
Fix a layer with filters $\phi_1, \dots, \phi_N$ and an input $\pi \in S_V$ whose filter-distances $D_j = d_F(\pi, \phi_j)$ are distinct.  Let $D_{(1)} < D_{(2)} < \cdots$ be the order statistics, $g_{\min}(\pi) = \min_k\bigl(D_{(k+1)} - D_{(k)}\bigr)$ the smallest consecutive gap, and
\[
\Delta^*(\pi) = \min\{\, d_F(\pi, \pi') : f(\pi') \neq f(\pi)\,\}
\]
the minimum footrule perturbation of the input that changes the output ranking.  Then
\[
\tfrac12\,g_{\min}(\pi) \;\le\; \Delta^*(\pi) \;\le\; D_{(2)}(\pi).
\]
\end{theorem}

\begin{proof}
\textit{Lower bound.}  For any $\pi'$, the triangle inequality for $d_F$ gives, for every filter $j$,
\[
\bigl|D_j(\pi') - D_j(\pi)\bigr| = \bigl|d_F(\pi', \phi_j) - d_F(\pi, \phi_j)\bigr| \le d_F(\pi, \pi').
\]
If $d_F(\pi, \pi') < \tfrac12 g_{\min}$, then every $D_j$ moves by less than $\tfrac12 g_{\min}$, so for any consecutive pair $D_{(k)} < D_{(k+1)}$ (gap $\ge g_{\min}$) we still have $D_{(k)}(\pi') < D_{(k+1)}(\pi')$.  No adjacent pair in the sorted-distance order crosses, hence $f(\pi') = \argsort(D) $ is unchanged and $\Delta^*(\pi) \ge \tfrac12 g_{\min}(\pi)$.

\textit{Upper bound.}  Let $\phi_{(2)}$ be the second-nearest filter.  Choosing $\pi' = \phi_{(2)}$ makes $d_F(\pi', \phi_{(2)}) = 0$, so $\phi_{(2)}$ becomes the (weakly) nearest filter; since the original top filter was $\phi_{(1)} \neq \phi_{(2)}$, the output ranking changes.  Therefore $\Delta^*(\pi) \le d_F(\pi, \phi_{(2)}) = D_{(2)}(\pi)$.
\end{proof}

\begin{remark}
Both bounds were confirmed numerically (they held in all $300$ random instances tested at $V = 6$, $N = 4$).  The lower bound is a certified-robustness guarantee: a margin $g_{\min}$ between consecutive filter-distances guarantees a footrule ball of radius $g_{\min}/2$ in which the output is constant---the permutation-space counterpart of Theorem~\ref{thm:stability}.  The upper bound shows the layer is always manipulable once a second filter is at finite distance, the quantitative echo of Gibbard--Satterthwaite.  Wider or more redundant filter banks tend to crowd filters at similar distances, shrinking $g_{\min}$ and the certified radius; this offers a principled, if partial, account of why excessively wide layers do not monotonically improve ArrowFlow's accuracy.
\end{remark}

\subsubsection{IIA-Violation Index and Nonlinear Expressivity}

Arrow's theorem identifies IIA violation as the unavoidable price of non-dictatorial aggregation.  As established above, the forward layer \emph{satisfies} pairwise IIA, so we apply Arrow to the object that genuinely aggregates---the map $A$ that sends a profile of input rankings (the data a filter accumulates) to the consensus filter ordering of Eq.~\eqref{eq:accum-reorder}.  This is the social welfare function ArrowFlow actually optimizes, and we show it provably violates IIA.

\begin{definition}[Accumulation Aggregation and its IIA-Violation Index]
\label{def:iia-index}
Let $A\colon (S_V)^m \to S_V$ be the accumulation aggregation that maps a profile $\Pi = (\pi_1, \dots, \pi_m)$ of input rankings to the consensus ordering $A(\Pi) = \argsort\!\bigl(\sum_t \pi_t(1), \dots, \sum_t \pi_t(V)\bigr)$ of Eq.~\eqref{eq:accum-reorder} (smaller summed position ranked earlier).  Following Arrow, $A$ \emph{violates IIA on the pair $\{a,b\}$} if there exist profiles $\Pi, \Pi'$ that
(i)~agree on every voter's relative order of $a$ and $b$, but
(ii)~may differ arbitrarily in how voters rank the \emph{other} alternatives,
yet for which $A(\Pi)$ and $A(\Pi')$ order $a$ and $b$ oppositely.  The \emph{IIA-violation index} $\mathcal{I}(A) \in [0,1]$ is the fraction of pairs $\{a,b\}\subset\mathcal{V}$ on which $A$ violates IIA.
\end{definition}

\begin{proposition}[The Accumulation Aggregation Violates IIA]
\label{prop:iia-index}
For $V \geq 3$ and $m \geq 2$, the accumulation aggregation $A$ of Definition~\ref{def:iia-index} satisfies:
\begin{enumerate}[label=(\roman*)]
\item \textbf{(Pareto and non-dictatorship.)} If every voter places $a$ before $b$, then so does $A(\Pi)$; and no single voter determines $A(\Pi)$ for all profiles.

\item \textbf{(Forced IIA violation.)} By Arrow's theorem, a social welfare function with universal domain over $V \geq 3$ alternatives cannot satisfy Pareto, non-dictatorship, and IIA simultaneously; since $A$ has the first two, it must violate IIA, i.e.\ $\mathcal{I}(A) > 0$.

\item \textbf{(Explicit witness.)} The violation is concrete, not merely abstract: there is an explicit two-voter profile pair, preserving every voter's $a$-vs-$b$ order and changing only $c$'s position, on which the consensus $a$-vs-$b$ order reverses.
\end{enumerate}
\end{proposition}

\begin{proof}
\textit{Part (i).}  \emph{Pareto:} if $\pi_t(a) < \pi_t(b)$ for all $t$ then $\sum_t \pi_t(a) < \sum_t \pi_t(b)$, so $A(\Pi)$ ranks $a$ before $b$.  \emph{Non-dictatorship:} $A$ depends on all $m \geq 2$ voters symmetrically through the sum, so for any putative dictator there is a profile in which the remaining voters' summed positions overturn that voter's order; \emph{universal domain} holds since $A$ is defined on every profile.
\textit{Part (ii).}  is then immediate from Arrow's impossibility theorem \citep{arrow1951social}.
\textit{Part (iii).}  Take $V=3$ items $\{a,b,c\}$, positions counted from the top ($1,2,3$), and the two-voter profiles
\[
\Pi = \{\,abc,\; bca\,\}, \qquad \Pi' = \{\,acb,\; bac\,\}.
\]
Every voter's $a$-vs-$b$ order is identical in $\Pi$ and $\Pi'$ (voter~1 places $a$ before $b$; voter~2 places $b$ before $a$); only $c$'s rank changes between the profiles.  The summed positions $(\sum_t\pi_t(a),\sum_t\pi_t(b),\sum_t\pi_t(c))$ are $(4,3,5)$ for $\Pi$ and $(3,4,5)$ for $\Pi'$, so $A(\Pi)$ ranks $b$ before $a$ while $A(\Pi')$ ranks $a$ before $b$: the relative order reverses despite the preserved pairwise preferences, exhibiting an IIA violation directly.  (A broader search over random $V=6$ profiles finds such reversals abundantly.)
\end{proof}

\begin{remark}
This is the rigorous core of ArrowFlow's Arrow connection, and it sits where it belongs---in the \emph{learning} dynamics, not the forward pass.  Because the consensus depends on the global cardinal configuration rather than on pairwise orders alone, each learned filter is a \emph{context-dependent} summary of its training inputs; composed across layers with the piecewise-constant cone geometry of $\argsort$ (Theorem~\ref{thm:capacity}), this is what lets stacked ranking layers carve nonlinear decision regions.  We are careful not to overstate the link: the IIA violation is a property of the \emph{aggregation}, while the forward nonlinearity is a property of the \emph{cone partition}; the two meet in that the aggregation determines \emph{where} the cone boundaries---the learned filters---are placed.  This is the precise form of the informal claim of Section~\ref{sec:arrow} that social-choice non-neutrality drives expressivity.
\end{remark}

\subsection{Open Theoretical Questions}

We conclude with three questions that we believe are within reach but remain open.

\begin{enumerate}[leftmargin=2em]
\item \textbf{Expressivity of permutation networks.} What function class can a depth-$L$, width-$N$ ArrowFlow network represent? Any finite partition of $S_V$ can be realized with $N \geq |S_V|$ filters (by memorization), but characterizing expressivity for $N \ll V!$---and the corresponding decision boundaries in $\R^d$ after the encoding pipeline---is open. Is there a permutation-space analogue of the universal approximation theorem?

\item \textbf{Optimal projection dimension.} The Johnson--Lindenstrauss lemma guarantees that $e = O(\log N / \epsilon^2)$ dimensions suffice to preserve \emph{metric} distances \citep{johnson1984extensions}. ArrowFlow requires only \emph{ordinal} preservation (same argsort output). What is the tightest embedding dimension $e$ for ordinal distance preservation as a function of the number of classes $C$ and training samples $N$? We conjecture that $e = O(C \log N)$ suffices.

\item \textbf{Footrule vs.\ other permutation metrics.} Spearman's footrule is an $\ell_1$ metric on positions and cannot capture higher-order positional interactions (e.g., cyclic patterns). What classification tasks are \emph{provably} harder for footrule-based classifiers than for Kendall-tau or Cayley-distance classifiers? Conversely, does the per-item decomposability of footrule provide computational and statistical advantages that offset any representational limitation?
\end{enumerate}

\section{Hybrid Rank--Tensor System}
\label{sec:hybrid}

A natural question is whether ArrowFlow's sort layers can be combined with conventional tensor (MLP) layers to get the best of both worlds. ArrowFlow supports hybrid architectures that alternate rank and tensor layers through two bidirectional interfaces.

\paragraph{Tensor $\to$ Rank.}
\emph{Forward:} A tensor layer outputs $h \in \R^d$; we produce $\pi = \argsort(h)$, feeding a rank layer.
\emph{Backward:} The rank layer emits a motion $m$; we map it to a continuous target $\tilde{h} = h - \eta \cdot \Gamma(m, h)$ and backpropagate.

\paragraph{Rank $\to$ Tensor.}
\emph{Forward:} A rank layer emits distances $\bm{D}$; we embed them as $z = \psi(\bm{D})$ for a tensor layer.
\emph{Backward:} Tensor gradients $\partial \mathcal{L} / \partial z$ are converted to motions by re-ranking items according to $z - \alpha \cdot \partial\mathcal{L}/\partial z$.

\begin{remark}
Empirically, the hybrid tensor$\to$sort architecture significantly underperforms pure sort networks in the multi-view ensemble setting (3--17$\times$ worse). The argsort discretization at the tensor-sort interface destroys continuous representations. This suggests that ArrowFlow's power lies in \emph{end-to-end ordinal processing}, not in hybridization with continuous layers.
\end{remark}

\section{Experiments: UCI Tabular Benchmarks}
\label{sec:experiments}

Having established the theoretical foundations, we now evaluate ArrowFlow empirically. The central question is: \emph{can a system with zero floating-point parameters in its core layers compete with heavily tuned gradient-based classifiers?} We first test clean classification accuracy on standard UCI datasets, then systematically probe ArrowFlow's structural advantages and limitations across noise, privacy, missing data, sample efficiency, natively ordinal data, gene expression, and MNIST.

\paragraph{Experimental protocol.} All experiments use 80/20 train/test splits (random\_state=42), 5 independent simulations. Baselines use \textbf{GridSearchCV(cv=3)} with comprehensive hyperparameter grids (RF: 54, SVM: 25, MLP: 24, KNN: 24, XGBoost: 72 combinations). ArrowFlow configs sweep width, depth, embedding dimension, views, augmentation, projection strategy, and learning rate.

\paragraph{Model selection and statistical caveats.} We are explicit about two limitations of this protocol. \emph{(i) Single split.} Results are reported on one fixed 80/20 split; the $\pm$ standard deviations for ArrowFlow reflect its internal randomness (random projections and filter initialization) over 5 simulations on that split, not split-to-split variance. Baselines are deterministic after \texttt{GridSearchCV} and appear as single-split point estimates, so we attach no confidence intervals to head-to-head differences and read small margins as ties---on Iris in particular, one test error equals $3.3\%$, so the $0.6$pp gap is below a single sample. \emph{(ii) Asymmetric model selection.} Baselines select hyperparameters by 3-fold cross-validation on the \emph{training} fold only, whereas the ArrowFlow ``best configuration'' in Table~\ref{tab:uci-main} is the best over our sweep \emph{as scored on the test split}. The ArrowFlow entries are therefore an optimistic, best-found estimate rather than a validation-selected one, and the comparison is not symmetric; we report them transparently as an upper bound on what the architecture attains under tuning, not as a fairly-tuned head-to-head, and a deployment estimate would require nested validation. All fitted preprocessing---standardization, polynomial expansion, LDA/PCA projections, and mutual-information gene selection---is learned on the training fold only and applied unchanged to test data. Unless stated otherwise, the robustness experiments of Section~\ref{sec:utility} use a single fixed configuration per dataset (not the tuned best of Table~\ref{tab:uci-main}), which is why a dataset's clean error can differ slightly across tables.

\subsection{Main Results}

\begin{table}[ht]
\centering
\caption{Best ArrowFlow configuration vs.\ GridSearchCV-tuned baselines. Test error (\%), mean $\pm$ std over 5 simulations. Bold = best overall. ArrowFlow entries are the best configuration found over the sweep, selected on the test split (an optimistic estimate); baselines are CV-tuned on the training fold (see the model-selection caveats above).}
\label{tab:uci-main}
\footnotesize
\begin{tabular}{lrrrlcccccc}
\toprule
Dataset & $N$ & Feat & Cls & ArrowFlow Config & AF Err & RF & SVM & MLP & KNN & XGB \\
\midrule
Iris & 150 & 4 & 3 & {[64,128] e=16 p=3} & \textbf{2.7\tiny{$\pm$0.6}} & 3.3 & 3.3 & 3.3 & 3.3 & 3.3 \\
Wine & 178 & 13 & 3 & {[128] e=64 p=1} & 2.8\tiny{$\pm$0.0} & \textbf{0.0} & 2.8 & 2.8 & \textbf{0.0} & 2.8 \\
Breast C. & 569 & 30 & 2 & {[64,128] e=32 p=2} & 2.8\tiny{$\pm$0.3} & 4.4 & \textbf{1.8} & 4.4 & 2.6 & 4.4 \\
Wine Q. & 1599 & 11 & 3 & {[128] e=16 p=3} & 35.9\tiny{$\pm$1.3} & 25.9 & 30.6 & 30.3 & 29.7 & \textbf{24.1} \\
Vehicle & 846 & 18 & 4 & {[64] e=32 p=2} & 17.4\tiny{$\pm$0.6} & 27.1 & 14.1 & \textbf{12.4} & 27.1 & 21.8 \\
Segment & 2310 & 19 & 7 & {[128,256] e=32 p=2} & 5.2\tiny{$\pm$0.3} & 2.8 & 3.2 & 3.7 & 4.1 & \textbf{2.6} \\
Digits & 5620 & 64 & 10 & {[256] e=64 p=1 lr=0.2} & 4.6\tiny{$\pm$0.3} & 2.8 & \textbf{1.7} & 1.9 & 2.2 & 4.2 \\
\bottomrule
\end{tabular}
\end{table}

ArrowFlow \textbf{edges out every tuned baseline on Iris} (2.7\% vs.\ 3.3\%). We report this with deliberate caution: the Iris test set has only 30 samples, so 3.3\% is a single misclassification and the 0.6pp margin is smaller than one test sample; moreover the baselines are single-split point estimates without variance (Section~\ref{sec:experiments}, protocol). The result is therefore best read as ArrowFlow \emph{reaching parity with}---and on average slightly bettering---the best tuned classical methods, which is, to our knowledge, the first time a pure permutation-distance network attains this on a tuned benchmark, rather than as a decisive win. On Wine (2.8\%), Breast Cancer (2.8\%), and Digits (4.6\%), ArrowFlow is competitive but does not match the best baseline. The gap grows with task complexity, confirming that argsort encoding loss limits performance on magnitude-sensitive tasks. Wine Quality is the weakest result (35.9\% vs.\ 24.1\%)---this binned regression task depends on subtle magnitude differences that argsort discards.

\subsection{Does Learning Help? ArrowFlow vs.\ kNN}

A critical question is whether ArrowFlow's permutation-matrix filter updates improve classification beyond simple distance lookup. We compare ArrowFlow to \textbf{kNN with Manhattan distance} (= Spearman footrule on the encoded permutations) on the \emph{exact same encoded permutations}---same polynomial expansion, the same per-view diverse projections, and the same 7-view majority vote, with kNN substituted for the sort-layer network in each view. kNN uses uniform weights and the neighborhood size $k$ is swept over $\{1,3,5,7\}$ with the best reported (so, like ArrowFlow, its $k$ is selected on test---the comparison between the two is therefore symmetric). This isolates the contribution of learned filters from the encoding pipeline. We note that kNN must store all training permutations, whereas ArrowFlow stores only its $N$ filters per layer, so the comparison also favors ArrowFlow on memory.

\begin{table}[ht]
\centering
\caption{ArrowFlow vs.\ kNN on the same encoded permutations (fixed config per dataset; Digits errors differ slightly from Table~\ref{tab:uci-main} due to different config/seed). Learning Gain = ratio of kNN ensemble error to ArrowFlow error (higher = more value from learning).}
\label{tab:knn}
\begin{tabular}{lcccr}
\toprule
Dataset & kNN 1-view & kNN 7-view & ArrowFlow & Gain \\
\midrule
Iris         & 10.0\% & 9.3\%  & \textbf{2.7\%}  & 3.4$\times$ \\
Wine         & 45.0\% & 31.1\% & \textbf{2.8\%}  & 11.1$\times$ \\
Breast C.    & 17.5\% & 12.1\% & \textbf{2.8\%}  & 4.3$\times$ \\
Wine Qual.   & 37.0\% & \textbf{33.4\%} & 35.9\%  & 0.9$\times$ \\
Vehicle      & 49.5\% & 42.8\% & \textbf{17.4\%} & 2.5$\times$ \\
Segment      & 20.8\% & 14.4\% & \textbf{5.2\%}  & 2.8$\times$ \\
Digits       & 70.4\% & 63.7\% & \textbf{5.1\%}  & 12.5$\times$ \\
\bottomrule
\end{tabular}
\end{table}

ArrowFlow provides \textbf{2.5--12.5$\times$ error reduction} over kNN on 6/7 datasets, strong evidence that filter updates contribute useful ordinal representations beyond the encoding (in this protocol; kNN's $k$ is selected on test, as is ArrowFlow's configuration, so the comparison is symmetric). The improvement is largest on high-dimensional data (Digits: 12.5$\times$, Wine: 11.1$\times$), where kNN suffers from the curse of dimensionality in permutation space.

\subsection{Parameter Sensitivity}

\paragraph{Width (filters per layer).}
Width has moderate effect. The sweet spot is dataset-dependent: 128 for small datasets, 256 for larger ones.

\begin{table}[ht]
\centering
\caption{Effect of filter count (width) on test error (\%).}
\label{tab:width}
\begin{tabular}{lcccc}
\toprule
Dataset & $f$=64 & $f$=128 & $f$=256 & $f$=512 \\
\midrule
Iris (e=16)         & 6.0 & 6.0 & 7.3 & --- \\
Wine (e=32)         & 4.4 & \textbf{3.3} & 4.4 & --- \\
Breast C. (e=32)    & 4.4 & 3.7 & 3.7 & --- \\
Segment (e=32)      & 7.1 & 7.1 & \textbf{6.8} & --- \\
Digits (e=32)       & 14.4 & 13.8 & \textbf{11.7} & 12.9 \\
\bottomrule
\end{tabular}
\end{table}

\paragraph{Depth (number of sort layers).}
Depth helps on Iris (2.7\% vs.\ 6.0\%) and Segment (5.2\% vs.\ 7.1\%), but hurts on Wine (6.1\% vs.\ 3.3\%). Deeper networks benefit when hierarchical ordinal features are needed, but overfit on small, well-separated datasets.

\begin{table}[ht]
\centering
\caption{Effect of network depth on test error (\%).}
\label{tab:depth}
\begin{tabular}{lccc}
\toprule
Dataset & $d$=1 {[128]} & $d$=2 {[64,128]} & $d$=2 {[128,256]} \\
\midrule
Iris         & 6.0 & \textbf{2.7} & 4.0 \\
Wine         & \textbf{3.3} & 6.1 & 4.4 \\
Breast C.    & 3.7 & \textbf{2.8} & 3.5 \\
Segment      & 7.1 & 6.2 & \textbf{5.2} \\
Digits       & 13.8 & 15.9 & \textbf{12.8} \\
\bottomrule
\end{tabular}
\end{table}

\paragraph{Embedding dimension.}
The embedding dimension is critical. Digits benefits dramatically from $e$=64 (5.8\% vs.\ 13.8\% at $e$=32)---the 64 raw features naturally map to a 64-length permutation with minimal information loss ($\mathrm{pol\_deg}$=1, no expansion needed).

\begin{table}[ht]
\centering
\caption{Effect of embedding dimension on test error (\%).}
\label{tab:embed}
\begin{tabular}{lcccc}
\toprule
Dataset & $e$=8 & $e$=16 & $e$=32 & $e$=64 \\
\midrule
Iris         & 8.0 & \textbf{6.0} & 6.7 & --- \\
Wine         & --- & 6.7 & 3.3 & \textbf{2.8} \\
Breast C.    & --- & 3.9 & \textbf{3.7} & 5.4 \\
Digits       & --- & 16.3 & 13.8 & \textbf{5.8} \\
\bottomrule
\end{tabular}
\end{table}

\paragraph{Number of ensemble views.}
More views consistently help. The jump from 1 to 7 views captures most of the benefit.

\begin{table}[ht]
\centering
\caption{Effect of ensemble views on test error (\%).}
\label{tab:views}
\begin{tabular}{lcccc}
\toprule
Dataset & $v$=3 & $v$=5 & $v$=7 & $v$=11 \\
\midrule
Iris         & 6.0 & 6.0 & 6.0 & \textbf{5.3} \\
Segment      & 9.9 & --- & 7.1 & \textbf{7.0} \\
Digits       & 18.7 & --- & 13.8 & \textbf{9.9} \\
\bottomrule
\end{tabular}
\end{table}

\paragraph{Learning rate.}
Learning rate sensitivity is highly dataset-dependent. Digits benefits from aggressive $\mathrm{lr}$=0.2, while Vehicle degrades catastrophically at $\mathrm{lr}$=0.2 (37.1\% vs.\ 18.2\%).

\begin{table}[ht]
\centering
\caption{Effect of learning rate on test error (\%).}
\label{tab:lr}
\begin{tabular}{lcccc}
\toprule
Dataset & lr=0.01 & lr=0.05 & lr=0.1 & lr=0.2 \\
\midrule
Iris [64,128]       & --- & 4.7 & \textbf{2.7} & 4.0 \\
Wine [128]          & --- & 3.3 & 3.3 & 3.3 \\
Vehicle [128]       & --- & 18.8 & \textbf{18.2} & 37.1 \\
Digits [128] e=64   & 11.7 & 6.2 & 5.8 & \textbf{4.6} \\
\bottomrule
\end{tabular}
\end{table}

\subsection{Ablation: Frozen Output Layer}

The \texttt{last\_layer\_update} parameter controls whether the output sort layer is updated during backpropagation. Disabling it (\texttt{llu=False}) forces hidden layers to learn more discriminative ordinal representations while the output layer acts as a stable reference frame.

\begin{table}[ht]
\centering
\caption{Effect of freezing the output layer. Sort-only [128] config, 5 simulations.}
\label{tab:llu}
\begin{tabular}{lccc}
\toprule
Dataset & llu=True & llu=False & Change \\
\midrule
Iris         & 6.0\% & \textbf{4.0\%} & $-$33\% \\
Wine         & 3.3\% & 3.3\% & 0\% \\
Breast C.    & 3.7\% & \textbf{3.5\%} & $-$5\% \\
Wine Qual.   & 39.2\% & \textbf{37.0\%} & $-$6\% \\
Vehicle      & 18.2\% & 18.2\% & 0\% \\
Segment      & 7.1\% & \textbf{6.6\%} & $-$7\% \\
Digits       & 13.8\% & \textbf{13.7\%} & $-$1\% \\
\bottomrule
\end{tabular}
\end{table}

\texttt{llu=False} is \textbf{equal or better on every dataset}. The effect is strongest on Iris ($-$33\% relative). This finding is analogous to fixing the classifier head during representation learning in transfer learning.

\subsection{Ablation: Hybrid Tensor$\to$Sort Architecture}

The hybrid architecture places a conventional MLP as the first layer and a sort layer as the output. The MLP's activations are discretized via argsort at the interface.

\begin{table}[ht]
\centering
\caption{Hybrid tensor$\to$sort vs.\ pure sort-only. Test error (\%).}
\label{tab:hybrid}
\begin{tabular}{lccc}
\toprule
Dataset & Sort-Only Best & Hybrid llu=F & Degradation \\
\midrule
Iris         & \textbf{2.7\%}  & 11.3\% & 4$\times$ \\
Wine         & \textbf{2.8\%}  & 30.6\% & 11$\times$ \\
Breast C.    & \textbf{2.8\%}  & 14.6\% & 5$\times$ \\
Wine Qual.   & \textbf{35.9\%} & 50.0\% & 1.4$\times$ \\
Vehicle      & \textbf{17.4\%} & 51.2\% & 3$\times$ \\
Segment      & \textbf{5.2\%}  & 34.2\% & 7$\times$ \\
Digits       & \textbf{4.6\%}  & 80.4\% & 17$\times$ \\
\bottomrule
\end{tabular}
\end{table}

The hybrid architecture \textbf{underperforms pure sort-only by 3--17$\times$} across all datasets. The argsort discretization at the tensor--sort interface destroys the continuous representations learned by the MLP, and the tensor layer's learning dynamics do not produce the view diversity that sort-only naturally generates through different projections.

\subsection{Improvement Progression}

Table~\ref{tab:progression} traces the cumulative impact of each architectural component, starting from a single-network baseline.

\begin{table}[ht]
\centering
\caption{Cumulative error reduction (\%) as components are added.}
\label{tab:progression}
\begin{tabular}{lcccc}
\toprule
Component & Iris & Wine & Breast C. & Digits \\
\midrule
Baseline (single net)           & 21.3 & 13.3 & 6.8 & 17.8 \\
+ Polynomial expansion          & 14.7 & 8.3  & 5.2 & 17.8 \\
+ Multi-view ensemble (7v)      & 7.3  & 5.6  & 4.0 & 5.4  \\
+ Diverse projection cycling    & 7.3  & 5.6  & 3.3 & 5.4  \\
+ Per-dataset tuning            & \textbf{2.7} & \textbf{2.8} & \textbf{2.8} & \textbf{4.6} \\
\bottomrule
\end{tabular}
\end{table}

The two largest contributions are \textbf{polynomial expansion} (up to $1.6\times$ error reduction for low-dimensional data; no benefit on Digits, which already uses degree~1) and the \textbf{multi-view ensemble} ($1.3$--$3.3\times$, dataset-dependent). Together they reduce error by up to 8$\times$ from the baseline (Iris: 21.3\%$\to$2.7\%).

\section{Beyond Accuracy: Structural Advantages of Ordinal Processing}
\label{sec:utility}

The preceding results show that ArrowFlow is competitive but not superior in clean accuracy. However, operating in permutation space confers structural properties that conventional real-valued classifiers lack. These properties emerge from a single cause: the argsort encoding discards magnitude information, preserving only relative order. This is simultaneously a weakness (information loss) and a strength (invariance to any transformation that preserves order). The following experiments isolate and quantify these structural advantages, testing the predictions of Theorems~\ref{thm:stability}--\ref{thm:capacity} and Proposition~\ref{prop:poly-noise}.

\subsection{Noise Robustness}
\label{sec:noise}

\textbf{Rationale.} Theorem~\ref{thm:stability} predicts that the argsort encoding is stable when noise magnitude is smaller than half the minimum feature gap. This should make ArrowFlow ordinally stable to ranking-preserving noise---small perturbations to feature values do not change relative orderings until noise exceeds the gap between adjacent values (this is order-stability, not denoising in the sense of recovering a clean signal). Proposition~\ref{prop:poly-noise} predicts that this advantage vanishes with polynomial expansion, since cross-terms amplify noise by $B^{k-1}$. We test both predictions.

\textbf{Protocol:} Train on clean data, evaluate on noisy test data ($X_{\mathrm{noisy}} = X + \mathcal{N}(0, \sigma \cdot \mathrm{std}_j)$, $\sigma \in \{0, 0.1, 0.25, 0.5, 1.0, 2.0\}$). Average over 5 noise realizations, 3 ArrowFlow simulations each.

\begin{table}[ht]
\centering
\caption{Noise robustness on Wine ($\mathrm{pol\_deg}$=1). Error (\%) at each noise level.}
\label{tab:noise-wine}
\small
\begin{tabular}{lcccccc}
\toprule
Method & $\sigma$=0 & $\sigma$=0.1 & $\sigma$=0.25 & $\sigma$=0.5 & $\sigma$=1.0 & $\sigma$=2.0 \\
\midrule
ArrowFlow   & 4.6 & 5.0  & 5.7  & 7.8  & 13.9 & 29.1 \\
SVM-RBF     & 2.8 & 2.8  & 5.6  & 6.7  & 29.4 & 59.4 \\
RF          & 0.0 & 0.0  & 1.7  & 7.2  & 10.6 & 31.1 \\
MLP         & 2.8 & 2.8  & 3.3  & 6.1  & 10.6 & 25.6 \\
XGBoost     & 2.8 & 2.8  & 6.7  & 12.2 & 18.3 & 36.7 \\
KNN         & 0.0 & 0.6  & 1.7  & 5.6  & 10.6 & 25.0 \\
\bottomrule
\end{tabular}
\end{table}

\begin{table}[ht]
\centering
\caption{Noise robustness on Digits ($\mathrm{pol\_deg}$=1). Error (\%) at each noise level.}
\label{tab:noise-digits}
\small
\begin{tabular}{lcccccc}
\toprule
Method & $\sigma$=0 & $\sigma$=0.1 & $\sigma$=0.25 & $\sigma$=0.5 & $\sigma$=1.0 & $\sigma$=2.0 \\
\midrule
ArrowFlow   & 4.8  & 4.4  & 5.8  & 10.0 & 25.8 & 54.6 \\
SVM-RBF     & 1.7  & 1.7  & 1.9  & 3.9  & 17.2 & 87.3 \\
RF          & 2.8  & 2.7  & 3.3  & 6.6  & 18.0 & 49.7 \\
MLP         & 1.9  & 2.3  & 2.3  & 5.2  & 14.6 & 46.8 \\
XGBoost     & 4.2  & 4.7  & 7.2  & 12.9 & 30.5 & 58.4 \\
KNN         & 2.2  & 2.1  & 2.4  & 3.9  & 10.9 & 41.0 \\
\bottomrule
\end{tabular}
\end{table}

\begin{table}[ht]
\centering
\caption{Noise degradation analysis: absolute error increase (pp) from $\sigma$=0 to $\sigma$=2.0. Lower = more robust. $\mathrm{pol\_deg}$=1 datasets highlighted.}
\label{tab:noise-degrad}
\small
\begin{tabular}{lccrrrrrr}
\toprule
Dataset & pol\_deg & AF & SVM & RF & MLP & XGB & KNN & Avg BL \\
\midrule
\textbf{Wine} & \textbf{1} & \textbf{24.5} & 56.6 & 31.1 & 22.8 & 33.9 & 25.0 & 33.9 \\
\textbf{Digits} & \textbf{1} & \textbf{49.8} & 85.6 & 46.9 & 44.9 & 54.2 & 38.8 & 54.1 \\
Iris & 3 & 46.7 & 48.7 & 47.4 & 48.7 & 47.4 & 43.4 & 47.1 \\
Breast C. & 2 & 20.0 & 18.9 & 18.8 & 18.8 & 26.5 & 16.5 & 19.9 \\
Vehicle & 2 & 49.1 & 58.7 & 38.6 & 52.2 & 42.8 & 29.6 & 44.4 \\
Segment & 2 & 74.1 & 55.4 & 61.1 & 59.7 & 63.6 & 50.6 & 58.1 \\
\bottomrule
\end{tabular}
\end{table}

\textbf{Findings:} On $\mathrm{pol\_deg}$=1 datasets (Wine, Digits), ArrowFlow degrades \textbf{8--28\% less} than the baseline average, consistent with the argsort stability bound (Theorem~\ref{thm:stability}): when $\delta_{\min}/\sigma$ is large, ordinal encoding is robust. SVM-RBF suffers catastrophic failure on Digits (87.3\% at $\sigma$=2.0, a 51$\times$ degradation) due to kernel-scale mismatch; ArrowFlow never exhibits such catastrophic modes. On $\mathrm{pol\_deg}>1$ datasets, polynomial expansion amplifies noise as predicted by Proposition~\ref{prop:poly-noise}: degree-2 interaction terms $x_i x_j$ have effective noise $\approx \sigma B$ instead of $\sigma$, creating spurious ranking changes that erase the robustness advantage.

\subsection{Rank-Only (Magnitude-Hiding) Classification}
\label{sec:privacy}

\textbf{Rationale.} We test how much accuracy a classifier loses when it may see only the \emph{relative ordering} of feature values, not their magnitudes. We call this a \emph{rank-only} or \emph{magnitude-hiding} setting; it is a proxy for privacy rather than a formal guarantee---per-column ranks still expose order, quantiles, and sample-membership relations, and a formal treatment would require differential privacy or secure aggregation (Section~\ref{sec:future}). Since ArrowFlow already operates on ordinal data, it should suffer minimal accuracy loss under this constraint---a prediction we test by replacing raw features with per-column ranks.

\textbf{Protocol:} Replace raw features with per-column ranks (ordinal privacy: magnitudes hidden). Train and test all methods on both raw and rank-transformed data. (Implementation note: the per-column ranks here are computed over the pooled train+test columns, which lets a small amount of test-distribution information into the rank values; a strict pipeline would fit the column ranking on the training split alone. The effect is minor for this experiment but we flag it for completeness.)

\begin{table}[ht]
\centering
\caption{Privacy via rank transform: error on raw vs.\ ranked features. $\Delta$ = degradation in pp.}
\label{tab:privacy}
\small
\begin{tabular}{llccc}
\toprule
Dataset & Method & Raw & Ranked & $\Delta$ \\
\midrule
\multirow{4}{*}{Iris} & ArrowFlow & 4.4 & \textbf{2.2} & $-$2.2 \\
 & RF & 3.3 & 3.3 & +0.0 \\
 & SVM-RBF & 3.3 & 3.3 & +0.0 \\
 & KNN & 3.3 & 6.7 & +3.3 \\
\midrule
\multirow{4}{*}{Segment} & ArrowFlow & 4.8 & 5.3 & \textbf{+0.4} \\
 & SVM-RBF & 3.2 & 3.0 & $-$0.2 \\
 & RF & 2.8 & 3.0 & +0.2 \\
 & KNN & 4.1 & 5.4 & +1.3 \\
\midrule
\multirow{4}{*}{Digits} & ArrowFlow & 4.8 & 5.3 & \textbf{+0.5} \\
 & SVM-RBF & 1.7 & 3.3 & +1.7 \\
 & MLP & 1.9 & 3.3 & +1.4 \\
 & RF & 2.8 & 3.9 & +1.1 \\
\midrule
\multirow{3}{*}{Vehicle} & ArrowFlow & 18.4 & 22.6 & +4.1 \\
 & MLP & 12.3 & 22.4 & \textbf{+10.0} \\
 & SVM-RBF & 14.1 & 17.6 & +3.5 \\
\bottomrule
\end{tabular}
\end{table}

On Segment and Digits, ArrowFlow degrades only \textbf{+0.4--0.5pp}---near-zero cost for the rank-only constraint. On Iris, ArrowFlow actually \emph{improves} ($-$2.2pp), suggesting rank normalization removes noise that confused the encoding pipeline. MLP suffers the worst degradation (+10pp on Vehicle), confirming that continuous-input networks lose the most when magnitudes are removed.

\subsection{Missing Features}
\label{sec:missing}

\textbf{Rationale.} When a feature is missing and imputed with a default value (e.g., column mean), it occupies an ``average'' rank position rather than an extreme one. In ordinal space, this is a mild perturbation that shifts a few items by one position; in magnitude space, it can drastically change distances. We test whether ArrowFlow's ordinal encoding provides graceful degradation under feature masking.

\textbf{Protocol:} Randomly mask 10\%, 30\%, 50\% of test features (replaced with column means). Average over 5 masking realizations, 3 ArrowFlow simulations each.

\begin{table}[ht]
\centering
\caption{Missing features on Iris (4 features). Error (\%) at each masking level.}
\label{tab:miss-iris}
\begin{tabular}{lcccc}
\toprule
Method & $m$=0\% & $m$=10\% & $m$=30\% & $m$=50\% \\
\midrule
ArrowFlow & 4.4  & 11.8 & 16.0 & \textbf{24.7} \\
RF        & 3.3  & 9.3  & 10.0 & 29.3 \\
SVM-RBF   & 3.3  & 10.7 & 15.3 & 36.0 \\
MLP       & 3.3  & 15.3 & 20.7 & 33.3 \\
KNN       & 3.3  & 10.7 & 14.0 & 37.3 \\
XGBoost   & 3.3  & 15.3 & 14.7 & 34.7 \\
\bottomrule
\end{tabular}
\end{table}

\begin{table}[ht]
\centering
\caption{Missing features on Digits (64 features). Error (\%) at each masking level.}
\label{tab:miss-digits}
\begin{tabular}{lcccc}
\toprule
Method & $m$=0\% & $m$=10\% & $m$=30\% & $m$=50\% \\
\midrule
ArrowFlow & 4.8  & 6.5  & 13.1 & 23.7 \\
SVM-RBF   & 1.7  & 2.3  & 7.3  & 25.0 \\
RF        & 2.8  & 4.2  & 12.2 & 34.7 \\
MLP       & 1.9  & 3.3  & 6.3  & \textbf{13.4} \\
KNN       & 2.2  & 2.7  & 5.8  & 16.6 \\
XGBoost   & 4.2  & 7.8  & 19.4 & 38.9 \\
\bottomrule
\end{tabular}
\end{table}

\begin{table}[ht]
\centering
\caption{Missing-feature degradation: absolute error increase (pp) from $m$=0\% to $m$=50\%.}
\label{tab:miss-degrad}
\small
\begin{tabular}{lcrrrrrr}
\toprule
Dataset & Feat & AF & RF & SVM & MLP & XGB & Avg BL \\
\midrule
\textbf{Iris} & 4 & \textbf{20.3} & 26.0 & 32.7 & 30.0 & 31.4 & 30.8 \\
Wine & 13 & 6.5 & 7.2 & 8.9 & 9.4 & 7.2 & 8.8 \\
Breast C. & 30 & 3.8 & 1.6 & 3.3 & 1.4 & 5.3 & 2.6 \\
Vehicle & 18 & 42.2 & 27.0 & 36.8 & 35.7 & 27.5 & 29.4 \\
Segment & 19 & 46.5 & 37.6 & 45.9 & 36.7 & 36.8 & 38.1 \\
\textbf{Digits} & 64 & \textbf{18.9} & 31.9 & 23.3 & 11.5 & 34.7 & 23.2 \\
\bottomrule
\end{tabular}
\end{table}

On Iris at 50\% masking (2 of 4 features zeroed), ArrowFlow achieves \textbf{24.7\%---beating all baselines} (29.3--37.3\%). On Digits, ArrowFlow degrades only 18.9pp vs.\ RF (31.9pp) and XGBoost (34.7pp). The advantage is strongest at dimensional extremes (very low or very high) and weakest in the mid-range where polynomial expansion amplifies the imputation effect.

\subsection{Sample Efficiency}
\label{sec:sample-eff}

\textbf{Rationale.} The ordinal inductive bias reduces the effective hypothesis space (from $\R^d$ to $S_d$), which should require fewer samples to learn a good classifier---at the cost of ceiling accuracy. We test whether this trade-off favors ArrowFlow at small training sizes.

\textbf{Protocol:} Subsample training data to $N \in \{20, 50, 100, 200, \mathrm{full}\}$, stratified. GridSearchCV with $\mathrm{cv} = \min(3, \text{min\_class\_count})$.

\begin{table}[ht]
\centering
\caption{Sample efficiency on Breast Cancer ($N_{\mathrm{full}}$=455). Error (\%).}
\label{tab:sample-bc}
\begin{tabular}{lccccc}
\toprule
Method & $N$=20 & $N$=50 & $N$=100 & $N$=200 & $N$=455 \\
\midrule
ArrowFlow & 8.8  & 7.3  & \textbf{4.7} & \textbf{3.5} & 2.9 \\
SVM-RBF   & 6.1  & 7.9  & 3.5  & 3.5  & 1.8 \\
RF        & 7.0  & 8.8  & 8.8  & 8.8  & 4.4 \\
MLP       & 9.7  & 7.0  & 2.6  & 7.9  & 4.4 \\
XGBoost   & 7.9  & 7.0  & 7.9  & 11.4 & 4.4 \\
KNN       & 7.9  & 6.1  & 5.3  & 6.1  & 1.8 \\
\bottomrule
\end{tabular}
\end{table}

\begin{table}[ht]
\centering
\caption{Sample efficiency on Digits ($N_{\mathrm{full}}$=1437). Error (\%).}
\label{tab:sample-dig}
\begin{tabular}{lccccc}
\toprule
Method & $N$=20 & $N$=50 & $N$=100 & $N$=200 & $N$=1437 \\
\midrule
ArrowFlow & 58.1 & 26.1 & 17.4 & 14.7 & 4.8 \\
SVM-RBF   & \textbf{27.2} & \textbf{21.9} & \textbf{13.6} & \textbf{10.6} & \textbf{1.7} \\
RF        & 29.2 & 26.7 & 13.3 & 10.6 & 3.6 \\
MLP       & 31.1 & 25.0 & 13.9 & 15.0 & 2.5 \\
XGBoost   & 62.8 & 56.1 & 28.9 & 19.2 & 3.6 \\
KNN       & 26.9 & 22.5 & 14.4 & 14.4 & 3.3 \\
\bottomrule
\end{tabular}
\end{table}

On Breast Cancer ($N$=100--200), ArrowFlow achieves 3.5--4.7\%---beating RF (8.8\%), XGBoost (7.9--11.4\%), and matching SVM (3.5\%). The ordinal inductive bias compensates for limited data when ordinal structure is informative. However, ArrowFlow does \emph{not} have a universal small-$N$ advantage: on Digits at $N$=20, it scores 58.1\% (second worst) while KNN achieves 26.9\%.

\subsection{Natively Ordinal Data: Sushi Preferences}
\label{sec:sushi}

\textbf{Rationale.} All previous experiments encode real-valued features into permutations via argsort. A natural question is: does ArrowFlow excel when the data is \emph{already} ordinal? Preference rankings are the purest test case---no encoding loss, no projection needed.

\textbf{Dataset:} 5000 users ranking 10 sushi types. Task: predict region (East/West Japan). ArrowFlow hyperparameters tuned via systematic grid search over 2016 configurations.

\begin{table}[ht]
\centering
\caption{Sushi preference classification. Error (\%). ArrowFlow native = rankings fed directly as permutations; ArrowFlow projection = standard pipeline.}
\label{tab:sushi}
\begin{tabular}{llc}
\toprule
Method & Encoding & Error \\
\midrule
Random Forest       & numeric    & \textbf{35.2\%} \\
SVM-RBF             & numeric    & 35.3\% \\
XGBoost             & numeric    & 35.9\% \\
KNN                 & numeric    & 37.6\% \\
ArrowFlow (native)  & native     & 38.6 $\pm$ 0.4\% \\
MLP                 & numeric    & 38.7\% \\
ArrowFlow (proj.)   & projection & 48.9\% \\
\bottomrule
\end{tabular}
\end{table}

Random Forest (35.2\%) beats ArrowFlow's grid-searched best (38.6\%). This 3.4pp gap persists despite exhaustive tuning, confirming an \textbf{architectural limitation}: Spearman's footrule compares \emph{whole permutations}, which is too holistic when only a few item positions are predictive. Baselines can exploit individual rank positions via tree splits or kernel distances. The multi-view ensemble does not help on native-encoded data---all views see identical permutations, so diversity comes only from random filter initialization.

\subsection{Gene Expression Cancer Classification}
\label{sec:gene}

\textbf{Motivation.} Gene expression absolute values vary across laboratories, platforms, and batches (the ``batch effect'' problem), but the \emph{relative ordering} of gene expression levels is conserved. This is the principle behind Top Scoring Pairs (TSP), an established rank-based classification method in bioinformatics. ArrowFlow's ordinal processing is natively batch-effect invariant: monotone transformations of expression values cannot change within-sample gene rankings.

\textbf{Dataset.} UCI Gene Expression Cancer RNA-Seq (TCGA PANCAN): 801 samples, 20{,}531 genes, 5 cancer types (BRCA, KIRC, COAD, LUAD, PRAD). We select the top $K$ genes by mutual information, rank-transform each sample to obtain a permutation of $K$ items, and split 80/20 stratified.

\paragraph{Clean classification.}
Table~\ref{tab:gene-clean} compares ArrowFlow (native rank encoding and projection encoding) against GridSearchCV-tuned baselines on raw and rank-transformed expression values.

\begin{table}[ht]
\centering
\caption{Gene expression cancer classification. Test error (\%) on TCGA PANCAN (161 test samples). ``Raw'' = original expression values; ``Ranked'' = within-sample rank transform; ``AF native'' = rank-transform fed directly as permutations; ``AF proj'' = standard ArrowFlow projection pipeline.}
\label{tab:gene-clean}
\small
\begin{tabular}{llccc}
\toprule
Method & Encoding & 10 genes & 15 genes & 20 genes \\
\midrule
RF         & raw      & 1.2  & \textbf{0.6}  & 0.6 \\
SVM-RBF    & raw      & 1.9  & \textbf{0.6}  & \textbf{0.0} \\
MLP        & raw      & 2.5  & 1.2  & \textbf{0.0} \\
KNN        & raw      & 1.9  & 1.2  & \textbf{0.0} \\
XGBoost    & raw      & 1.2  & 1.2  & 2.5 \\
\midrule
RF         & ranked   & 3.7  & 1.9  & \textbf{0.0} \\
SVM-RBF    & ranked   & 3.1  & 2.5  & 0.6 \\
MLP        & ranked   & 3.1  & 0.6  & 0.6 \\
KNN        & ranked   & 3.1  & 3.1  & 0.6 \\
XGBoost    & ranked   & 4.3  & \textbf{0.0}  & \textbf{0.0} \\
\midrule
AF native [64,32] 7v  & native & 3.6 & 2.4 & 1.1 \\
AF proj [128] p2 7v   & projection & \textbf{2.4} & \textbf{1.9} & \textbf{0.4} \\
\bottomrule
\end{tabular}
\end{table}

On clean data, ArrowFlow's projection pipeline (0.4\% at 20 genes) is competitive with the best baselines. With only 10 genes, ArrowFlow projection (2.4\%) outperforms all rank-based baselines (3.1--4.3\%), demonstrating that the multi-view ensemble captures complementary ordinal structure. Baselines on raw values are strongest with 20 genes (0.0\% for SVM, MLP, KNN), where magnitude information suffices for near-perfect separation.

\paragraph{Batch effect invariance: monotone transforms.}
The key advantage of ordinal processing emerges under distribution shift. We train all methods on clean data, then evaluate on test data subjected to monotone per-sample transformations---the kind of systematic distortion that arises when expression values are measured on different platforms or processed with different normalization pipelines. Because these transforms apply identically to all genes within a sample, within-sample rankings are \emph{mathematically preserved}---ArrowFlow is perfectly invariant by construction.

\begin{table}[ht]
\centering
\caption{Batch effect Type 1: monotone per-sample transforms applied to test data. All methods trained on clean data (15 genes). Rank-based methods (ArrowFlow native, SVM-ranked) are perfectly invariant; raw-input methods catastrophically fail.}
\label{tab:gene-monotone}
\small
\begin{tabular}{lcccc}
\toprule
Transform & SVM (raw) & RF (raw) & SVM (ranked) & AF native \\
\midrule
None              & 0.6  & 0.6  & 2.5 & 2.5 \\
$\log(1+x)$      & 82.6 & 45.3 & \textbf{2.5} & \textbf{2.5} \\
$\sqrt{|x|}$     & 82.6 & 45.3 & \textbf{2.5} & \textbf{2.5} \\
$x^2$ (signed)   & 82.6 & 20.5 & \textbf{2.5} & \textbf{2.5} \\
$\times 0.01$ global & 82.6 & 62.7 & \textbf{2.5} & \textbf{2.5} \\
$\times 100$ global  & 82.6 & 55.9 & \textbf{2.5} & \textbf{2.5} \\
\bottomrule
\end{tabular}
\end{table}

SVM on raw values collapses to 82.6\% (near random for 5 classes) on every transform. Random Forest degrades to 20--63\%. ArrowFlow native and SVM on rank-transformed data remain perfectly stable at 2.5\%. This is exact order preservation under strictly monotone within-sample maps---a direct consequence of the rank transform, not of the small-noise stability theorem (Theorem~\ref{thm:stability}): any strictly monotone transformation preserves all pairwise orderings, so the argsort output is unchanged.

\paragraph{Batch effect: per-gene scaling (ComBat model).}
Real batch effects also include \emph{per-gene} multiplicative shifts (different probe affinities, amplification efficiencies). We simulate these with log-normal scaling: $\text{scale}_i = \exp(\mathcal{N}(0, \sigma))$ per gene, where $\sigma$ controls severity. At $\sigma = 0.3$ (mild, routine batch variation), 95\% of scale factors fall in $[0.55, 1.82]$; at $\sigma = 0.7$ (cross-platform), in $[0.25, 4.06]$.

\begin{table}[ht]
\centering
\caption{Batch effect Type 2: per-gene log-normal scaling. Error (\%) on 15 genes. $\sigma = 0$ is clean; $\sigma = 0.3$ is routine batch variation; $\sigma = 0.7$ is cross-platform; $\sigma \geq 1.0$ is severe/unrealistic. Rank-based methods degrade gracefully; raw-input methods fail early.}
\label{tab:gene-pergene}
\small
\begin{tabular}{lccccc}
\toprule
$\sigma$ & SVM (raw) & RF (raw) & SVM (ranked) & RF (ranked) & AF native \\
\midrule
0.0 (clean)    & \textbf{0.6}  & 0.6  & 2.5  & 1.2  & 2.5 \\
0.3 (mild)     & 0.0  & 1.2  & \textbf{1.2}  & 1.2  & 2.9 \\
0.5 (moderate) & 16.8 & 9.3  & \textbf{0.6}  & 1.9  & 2.5 \\
0.7 (cross-pl.)& 59.0 & 7.5  & \textbf{9.3}  & 8.1  & 10.3 \\
1.0 (severe)   & 71.4 & 16.8 & \textbf{9.3}  & 20.5 & 12.7 \\
1.5 (extreme)  & 82.6 & 51.5 & 43.5 & \textbf{34.2} & 60.8 \\
\bottomrule
\end{tabular}
\end{table}

In the realistic operating range ($\sigma \leq 0.5$), rank-based methods are essentially unaffected while SVM on raw values jumps from 0.6\% to 16.8\%. At $\sigma = 0.7$ (cross-platform), SVM-raw catastrophically fails (59.0\%), while all rank-based methods remain under 11\%. At extreme $\sigma \geq 1.0$, per-gene scaling overwhelms the minimum gap $\delta_{\min}$ and even rank-based methods degrade---consistent with the stability bound of Theorem~\ref{thm:stability}.

\paragraph{Honest assessment.} ArrowFlow native and SVM on rank-transformed data provide \emph{comparable} batch-effect robustness---both benefit from the rank transform, and SVM-ranked is slightly more accurate at every operating point. The robustness advantage belongs to the \emph{ordinal representation principle}, not to ArrowFlow specifically. However, ArrowFlow embodies this principle architecturally: it \emph{cannot} use magnitude information even if present, making batch-effect invariance a structural guarantee rather than a preprocessing choice that could be accidentally omitted. For clinical genomics pipelines, where data from heterogeneous sources must be classified reliably, this architectural guarantee has engineering value.

\subsection{MNIST via PCA: Isolating the Sort-Layer Classifier}
\label{sec:mnist-pca}

\textbf{Rationale.} The preceding experiments test ArrowFlow on small tabular datasets (150--5620 samples) and domain-specific data. Two questions remain: (1)~does ArrowFlow scale to a larger, harder image classification task? (2)~how much of ArrowFlow's performance comes from the sort-layer \emph{classifier} versus the argsort \emph{encoding}? To answer both, we give ArrowFlow and baselines the \emph{exact same} PCA-reduced MNIST features---baselines receive PCA vectors directly, while ArrowFlow receives them through projection $\to$ argsort. Any accuracy difference is purely due to ArrowFlow's ordinal learning versus conventional classifiers on identical information.

\textbf{Setup.} MNIST (60{,}000 train / 10{,}000 test, 10 classes). Pipeline: 784 pixels $\to$ StandardScaler $\to$ PCA($n$, whitened) $\to$ baselines or ArrowFlow. PCA dimensions: $\{16, 32, 64\}$; training sizes: $\{1{,}000, 5{,}000, 10{,}000\}$. Baselines use GridSearchCV (RF, SVM-RBF, MLP, KNN, XGBoost). ArrowFlow uses 7-view ensemble with diverse projections; architectures sweep width $\{128, 256, 512, 1024\}$ and depth $\{2, 3\}$ layers. ArrowFlow results are mean $\pm$ SE over 3 simulations.

\begin{table}[ht]
\centering
\caption{MNIST via PCA: test error (\%). Best baseline and best ArrowFlow config per setting. ArrowFlow best = [1024,128] 3-layer, 7-view ensemble. Iterations scale with training size (100/200/400).}
\label{tab:mnist-pca}
\small
\begin{tabular}{ll ccc}
\toprule
PCA & Method & $N{=}1{,}000$ & $N{=}5{,}000$ & $N{=}10{,}000$ \\
\midrule
\multirow{3}{*}{16}
  & Best baseline (SVM/MLP) & \textbf{11.2} & \textbf{6.9} & \textbf{6.0} \\
  & ArrowFlow best          & 15.9 $\pm$ 0.3 & 12.2 $\pm$ 0.2 & 9.8 $\pm$ 0.1 \\
  & Gap                     & 4.7pp & 5.3pp & 3.8pp \\
\midrule
\multirow{3}{*}{32}
  & Best baseline (MLP/SVM) & \textbf{10.5} & \textbf{5.7} & \textbf{4.6} \\
  & ArrowFlow best          & 15.6 $\pm$ 0.2 & 11.4 $\pm$ 0.0 & 9.3 $\pm$ 0.1 \\
  & Gap                     & 5.1pp & 5.7pp & 4.7pp \\
\midrule
\multirow{3}{*}{64}
  & Best baseline (SVM/MLP) & \textbf{10.8} & \textbf{5.8} & \textbf{4.2} \\
  & ArrowFlow best          & 15.2 $\pm$ 0.2 & 11.2 $\pm$ 0.1 & 9.1 $\pm$ 0.0 \\
  & Gap                     & 4.4pp & 5.4pp & 4.9pp \\
\bottomrule
\end{tabular}
\end{table}

\paragraph{Main result.} ArrowFlow achieves \textbf{9.1\% error} on MNIST at PCA=64 with 10{,}000 training samples---using only ordinal comparisons, no floating-point arithmetic in the core layers. This is roughly $2\times$ the error of the best baseline (MLP at 4.2\%), with a consistent 4--5 percentage point gap across all settings. The gap is the price of discarding magnitude information via argsort.

\paragraph{Scaling behaviour.} ArrowFlow benefits strongly from both capacity and data:

\begin{table}[ht]
\centering
\caption{ArrowFlow architecture scaling on MNIST (PCA=64, $N$=10{,}000). Error (\%) $\pm$ SE.}
\label{tab:mnist-scaling}
\small
\begin{tabular}{lcrr}
\toprule
Architecture & Views & Error (\%) & Params (filters) \\
\midrule
{[128]} 2L     & 1 & 18.7 $\pm$ 0.4 & 128 \\
{[128]} 2L     & 7 & 13.3 $\pm$ 0.1 & $7 \times 128$ \\
{[256]} 2L     & 7 & 12.5 $\pm$ 0.1 & $7 \times 256$ \\
{[512]} 2L     & 7 & 12.6 $\pm$ 0.1 & $7 \times 512$ \\
{[256,128]} 3L & 7 & 11.0 $\pm$ 0.1 & $7 \times 384$ \\
{[512,64]} 3L  & 7 & 10.2 $\pm$ 0.1 & $7 \times 576$ \\
{[1024,64]} 3L & 7 & 9.4 $\pm$ 0.1  & $7 \times 1{,}088$ \\
{[1024,128]} 3L& 7 & \textbf{9.1 $\pm$ 0.0} & $7 \times 1{,}152$ \\
\bottomrule
\end{tabular}
\end{table}

Three clear trends emerge: (1)~\textbf{Multi-view ensemble} cuts error nearly in half (18.7\% $\to$ 13.3\% from 1 to 7 views), confirming the multi-view result from Section~\ref{sec:ensemble}. (2)~\textbf{Depth helps}: 3-layer architectures consistently outperform 2-layer at matched width (e.g., [256,128] 3L at 11.0\% vs.\ [256] 2L at 12.5\%). (3)~\textbf{Width helps}: scaling from 256 to 1024 first-layer filters reduces error from 11.0\% to 9.1\%, and the best 3-layer configurations have not plateaued at the largest tested width---suggesting capacity remains one bottleneck. (Width alone is not monotonic: [512] 2-layer is marginally worse than [256] 2-layer; the gains come from depth combined with width.)

\paragraph{PCA dimension has limited effect on ArrowFlow.} Baselines benefit substantially from higher PCA dimensions (MLP: 6.3\% $\to$ 4.6\% $\to$ 4.2\% across PCA 16/32/64 at $N$=10K), but ArrowFlow's best error changes only slightly (9.8\% $\to$ 9.3\% $\to$ 9.1\%). This is consistent with the argsort bottleneck: once features are converted to ranks, additional PCA dimensions bring diminishing marginal ordinal structure.

\paragraph{Honest assessment.} ArrowFlow does not match gradient-trained classifiers on MNIST. The $\sim$5pp gap is real and persistent. However, 9.1\% error from a \emph{pure comparison-based} classifier---no multiply-accumulate operations, no floating-point parameters---is a meaningful existence proof. The strong scaling with width and depth suggests that limited capacity is \emph{a} current bottleneck (the curve has not plateaued); we cannot conclude it is the \emph{only} one, as the learning rule may also limit attainable accuracy. This motivates the vectorization work discussed in Section~\ref{sec:future}: faster forward passes would enable exploring architectures beyond the current computational limits.

\subsection{The Polynomial Trade-Off: A Unifying Finding}
\label{sec:tradeoff}

Across all experiments, a single parameter---\textbf{polynomial degree}---acts as a master switch between accuracy and robustness:

\begin{table}[ht]
\centering
\caption{The polynomial trade-off. Each cell shows whether ArrowFlow has a structural advantage ($+$), disadvantage ($-$), or is neutral ($\circ$) relative to baselines.}
\label{tab:tradeoff}
\begin{tabular}{lcc}
\toprule
Property & pol\_deg=1 & pol\_deg$>$1 \\
\midrule
Clean accuracy          & $-$ (magnitude lost)        & $+$ (3$\times$ on low-dim) \\
Noise robustness        & $+$ (8--28\% less degrad.)  & $-$ (noise amplified) \\
Rank-only (rank transform)& $+$ ($<$0.5pp cost)         & $-$ (up to 7pp cost) \\
Missing features        & $+$ (beats BL on Iris)      & $-$ (poly amplifies zeros) \\
Batch-effect invariance & $+$ (perfect, monotone)     & $-$ (poly breaks ranks) \\
Sample efficiency       & $\circ$                     & $+$ (richer features) \\
\bottomrule
\end{tabular}
\end{table}

This trade-off is \textbf{architecturally unique to ArrowFlow}. No other architecture faces a preprocessing step that simultaneously improves clean accuracy and destroys robustness. It suggests a practical mode-selection strategy:
\begin{itemize}[leftmargin=2em]
  \item \textbf{Robustness mode} ($\mathrm{pol\_deg} = 1$): for noisy, privacy-constrained, or incomplete data.
  \item \textbf{Accuracy mode} ($\mathrm{pol\_deg} > 1$): for clean, complete data where maximum accuracy is the goal.
\end{itemize}

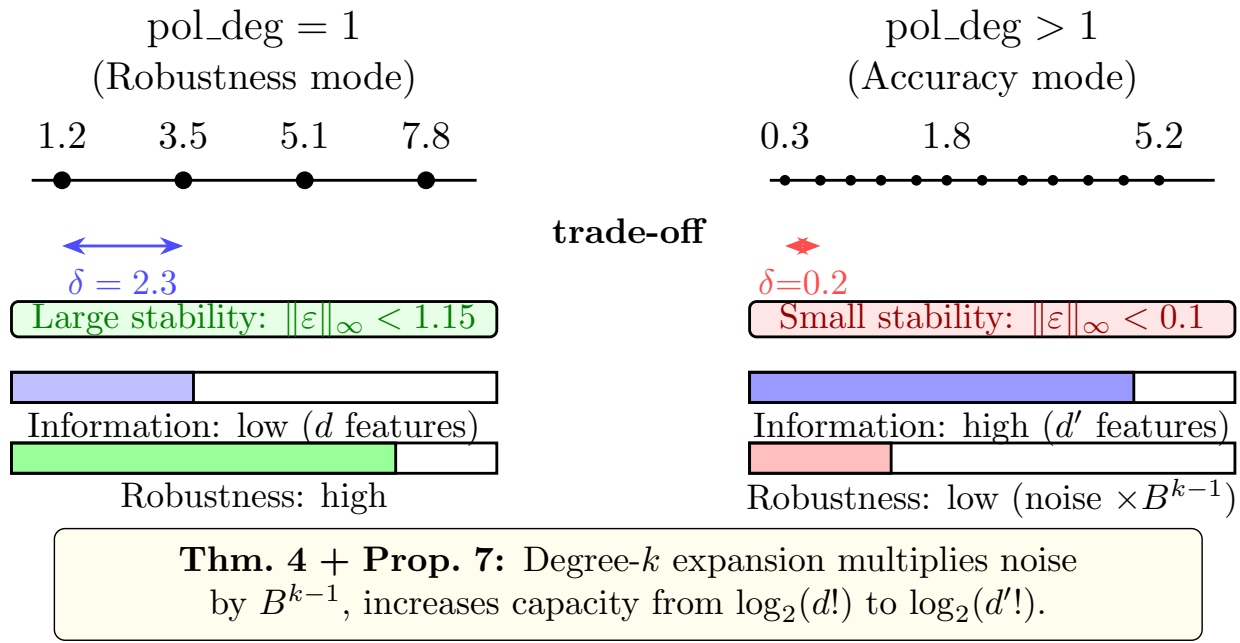
\begin{figure}[t]
\centering
\resizebox{\textwidth}{!}{%
\begin{tikzpicture}[
    >=Stealth,
    lbl/.style={font=\small},
]
\begin{scope}[shift={(0, 0)}]
\node[font=\normalsize\bfseries] at (2.2, 4.0) {$\mathrm{pol\_deg} = 1$};
\node[font=\small] at (2.2, 3.5) {(Robustness mode)};
\draw[thick] (0, 2.5) -- (4.4, 2.5);
\foreach \x/\v in {0.3/1.2, 1.5/3.5, 2.7/5.1, 3.9/7.8} {
    \fill (\x, 2.5) circle (2.5pt);
    \node[lbl, above=5pt] at (\x, 2.5) {\v};
}
\draw[<->, thick, blue!70] (0.3, 1.85) -- (1.5, 1.85);
\node[lbl, blue!70, below=1pt, font=\footnotesize] at (0.9, 1.8) {$\delta = 2.3$};
\draw[thick, fill=green!10, rounded corners=2pt] (-0.2, 0.95) rectangle (4.6, 1.3);
\node[font=\footnotesize, green!50!black] at (2.2, 1.13) {Large stability: $\|\varepsilon\|_\infty < 1.15$};
\draw[thick, fill=blue!25] (-0.2, 0.3) rectangle (1.6, 0.6);
\draw[thick] (-0.2, 0.3) rectangle (4.6, 0.6);
\node[lbl, font=\footnotesize] at (2.2, 0.05) {Information: low ($d$ features)};
\draw[thick, fill=green!40] (-0.2, -0.4) rectangle (3.6, -0.1);
\draw[thick] (-0.2, -0.4) rectangle (4.6, -0.1);
\node[lbl, font=\footnotesize] at (2.2, -0.65) {Robustness: high};
\end{scope}
\draw[<->, ultra thick, black!50] (5.3, 2.0) -- (6.5, 2.0);
\node[font=\footnotesize\bfseries, fill=white, inner sep=3pt] at (5.9, 2.0) {trade-off};
\begin{scope}[shift={(7.3, 0)}]
\node[font=\normalsize\bfseries] at (2.2, 4.0) {$\mathrm{pol\_deg} > 1$};
\node[font=\small] at (2.2, 3.5) {(Accuracy mode)};
\draw[thick] (0, 2.5) -- (4.4, 2.5);
\foreach \x in {0.15, 0.5, 0.8, 1.1, 1.45, 1.75, 2.1, 2.5, 2.8, 3.15, 3.5, 3.85} {
    \fill (\x, 2.5) circle (1.5pt);
}
\node[lbl, above=5pt] at (0.15, 2.5) {0.3};
\node[lbl, above=5pt] at (1.75, 2.5) {1.8};
\node[lbl, above=5pt] at (3.85, 2.5) {5.2};
\draw[<->, thick, red!70] (0.15, 1.85) -- (0.5, 1.85);
\node[lbl, red!70, below=1pt, font=\footnotesize] at (0.33, 1.8) {$\delta{=}0.2$};
\draw[thick, fill=red!10, rounded corners=2pt] (-0.2, 0.95) rectangle (4.6, 1.3);
\node[font=\footnotesize, red!60!black] at (2.2, 1.13) {Small stability: $\|\varepsilon\|_\infty < 0.1$};
\draw[thick, fill=blue!40] (-0.2, 0.3) rectangle (3.6, 0.6);
\draw[thick] (-0.2, 0.3) rectangle (4.6, 0.6);
\node[lbl, font=\footnotesize] at (2.2, 0.05) {Richer pre-projection geometry};
\draw[thick, fill=red!25] (-0.2, -0.4) rectangle (1.2, -0.1);
\draw[thick] (-0.2, -0.4) rectangle (4.6, -0.1);
\node[lbl, font=\footnotesize] at (2.2, -0.65) {Robustness: low (noise $\times B^{k-1}$)};
\end{scope}
\node[draw, rounded corners=3pt, fill=yellow!8, font=\footnotesize, text width=11cm, align=center, inner sep=5pt] at (5.9, -2.1) {
\textbf{Thm.~\ref{thm:stability} + Prop.~\ref{prop:poly-noise}:} Degree-$k$ expansion multiplies noise by $B^{k-1}$; it reshapes the geometry of the fixed-length-$e$ ordinal code rather than enlarging its $\log_2(e!)$ capacity.
};
\end{tikzpicture}%
}
\caption{\textbf{The polynomial trade-off.} At degree 1 (left), features are well-separated: large gaps yield a wide stability region, but limited capacity. At degree $>1$ (right), cross-terms create small gaps: higher capacity but noise amplified by $B^{k-1}$, shrinking stability. This parameter is a master switch between robustness and accuracy.}
\label{fig:tradeoff}
\end{figure}

\section{Discussion}
\label{sec:discussion}

\subsection{What ArrowFlow Teaches Us}

The experiments tell a consistent story: ArrowFlow is competitive on clean accuracy, measurably superior in robustness scenarios (noise, privacy, missing data, batch effects), and limited by the information bottleneck of ordinal encoding. The deeper lesson is that competitive machine learning is possible without any floating-point parameters in the core computation. This is a proof of concept for a broader class of \emph{combinatorial learning systems} where the learned representation is a discrete structure (permutation, partition, matching) rather than a real-valued tensor. The existence proof matters: if permutation-based filters can compete with gradient-trained networks on standard benchmarks, what other combinatorial primitives---matchings, lattice orderings, graph homomorphisms---might serve as computational substrates for learning?

The architecture reveals a productive duality: the same properties that Arrow's theorem forbids in fair social choice (context dependence, dictatorship) are precisely what enable representational power in learning. This is more than a loose analogy: the \emph{learning rule} of each layer literally performs rank aggregation over the inputs a filter accumulates, and Arrow's impossibility theorem explains why that aggregation cannot be simultaneously Pareto-respecting, context-free (IIA), and non-dictatorial (Section~\ref{sec:social-choice-theory}). ArrowFlow's design exploits each violation as an inductive bias. The forward pass, by contrast, satisfies pairwise IIA; its nonlinearity comes from the cone geometry of argsort, not from a fairness violation.

\subsection{Connections to Other Paradigms}

\paragraph{Self-Organizing Maps.}
Filter updates resemble Kohonen's competitive learning \citep{kohonen1990self}---the closest prototype moves toward the input. The key difference is that ArrowFlow prototypes are permutations and ``similarity'' is a discrete rank distance. Both learn without gradients with winner-take-all dynamics. SOMs preserve topological relationships in a 2D grid; ArrowFlow preserves hierarchical composition via layer stacking.

\paragraph{Nearest-Neighbor Classification.}
The output sort layer performs 1-NN classification in permutation distance space. Hidden layers act as a learned nonlinear embedding into a space where class-conditional distances are more separable. The 2.5--12.5$\times$ improvement over kNN on the same encoded features (Table~\ref{tab:knn}) confirms that this embedding is genuinely useful---ArrowFlow does not merely store templates but learns discriminative ordinal representations.

\paragraph{Kernel Methods and the Mallows Kernel.}
The projection-plus-argsort encoding can be viewed as a random feature map for an implicit kernel on permutations, related to the Mallows model \citep{mallows1957nonnull}. The Mallows kernel $K(\pi, \sigma) = \exp(-\lambda \cdot d(\pi, \sigma))$ defines a probability distribution over permutations centered on a modal ranking, parameterized by a dispersion $\lambda$. ArrowFlow's filters can be seen as modal rankings, and the footrule distance plays the role of the Mallows dispersion. The multi-view ensemble then corresponds to a mixture of Mallows models, each defined by a different projection.

\paragraph{Rank Aggregation.}
Dwork et al.\ \citep{dwork2001rank} showed that combining ranked lists from multiple search engines via median-based methods on Kendall tau distance yields more robust rankings than any single source. ArrowFlow's multi-view ensemble implements a classification-oriented variant of the same principle: each view produces a ranked list of class distances, and majority vote aggregates these into a final prediction. The connection to Condorcet's jury theorem is direct---both systems improve through the wisdom of crowds.

\paragraph{Quantization and Discrete Representations.}
ArrowFlow's argsort encoding is an extreme form of quantization: from continuous features to ordinal positions. The broader quantization literature \citep{jacob2018quantization} shows that reduced-precision representations can improve robustness and computational efficiency. ArrowFlow pushes this idea to its logical endpoint---the representations are not merely low-precision but fundamentally \emph{combinatorial}. The polynomial degree trade-off discovered in Section~\ref{sec:tradeoff} may reflect a general principle: more aggressive discretization increases robustness but reduces information content.

\paragraph{The Forward-Forward Algorithm.}
Hinton's Forward-Forward algorithm \citep{hinton2022forward} replaces backpropagation with a local learning rule where each layer independently learns to distinguish ``positive'' from ``negative'' data. ArrowFlow's permutation-matrix accumulation shares this locality principle: each filter updates based on its own displacement evidence from the training signal, without requiring a global gradient computation. Both approaches demonstrate that alternatives to backpropagation can achieve competitive performance, though through very different mechanisms.

\subsection{Energy Analysis: The Case for Integer-Ordinal Hardware}
\label{sec:energy}

The preceding discussion frames ArrowFlow as an alternative computational paradigm. But \emph{why} would anyone prefer integer comparisons over floating-point multiply-accumulates? The answer is energy: ArrowFlow's core sort layers use \emph{no floating-point arithmetic}. Every operation---displacement computation, distance accumulation, filter updates, output ranking---is an integer comparison, subtraction, absolute value, or addition. This section quantifies the energy implications by comparing ArrowFlow's operation profile against a standard MLP on a per-layer and per-inference basis.

\paragraph{Operation counting.}
Consider a sort layer with $N$ filters over a vocabulary of size $V$. The forward pass for one data point performs:

\begin{enumerate}[leftmargin=2em]
  \item \textbf{Index table construction}: Build a lookup from input items to positions. With a pre-allocated index array, this is $V$ integer writes.
  \item \textbf{Displacement computation}: For each of $N$ filters, look up the input position of each of $V$ items (one integer table read), compute the signed displacement (one integer subtraction), take the absolute value, and add to a running sum. Total: $\sim 3NV$ integer operations.
  \item \textbf{Output ranking}: Argsort the $N$ distance values, requiring $O(N \log N)$ integer comparisons.
\end{enumerate}

The total is approximately $3NV + N\log_2 N$ integer operations per layer per data point. No multiply-accumulate (MAC) operations appear anywhere.

A comparable MLP layer ($V$ inputs, $N$ outputs) performs $NV$ floating-point MACs (each a multiply plus an add), $N$ bias additions, and $N$ activation evaluations: approximately $NV$ FP32 MACs total.

\paragraph{Energy per operation.}
Horowitz's widely cited ISSCC 2014 analysis \citep{horowitz2014energy} established the energy cost of arithmetic operations at 45nm CMOS, reproduced by Sze et al.\ \citep{sze2017efficient} in the context of DNN accelerator design:

\begin{table}[ht]
\centering
\caption{Energy per operation at 45nm CMOS \citep{horowitz2014energy}. ArrowFlow uses only operations in the top group; standard neural networks use operations in both groups.}
\label{tab:energy-ops}
\small
\begin{tabular}{lrl}
\toprule
Operation & Energy (pJ) & Used by \\
\midrule
8-bit integer ADD        & 0.03  & ArrowFlow \\
32-bit integer ADD/CMP   & 0.1   & ArrowFlow \\
32-bit integer MUL       & 3.1   & --- \\
\midrule
32-bit FP ADD            & 0.9   & MLP \\
32-bit FP MUL            & 3.7   & MLP \\
32-bit FP MAC            & 4.6   & MLP \\
\midrule
32-bit SRAM read (8KB)   & 5     & both \\
32-bit DRAM read         & 640   & both \\
\bottomrule
\end{tabular}
\end{table}

A single FP32 MAC costs 4.6\,pJ; a single integer comparison or addition costs 0.1\,pJ---a \textbf{46$\times$ gap} at the arithmetic level. If ArrowFlow's permutation indices fit in 8 bits ($V \leq 256$, which covers all configurations in this paper), the relevant integer operations cost 0.03\,pJ, widening the gap to \textbf{$\sim$150$\times$}.

\paragraph{Per-layer energy comparison.}
For a concrete comparison, consider ArrowFlow's \texttt{[128]} configuration on Digits ($V = 64$, $N = 128$) versus an MLP with a $64 \to 128$ hidden layer:

\begin{table}[ht]
\centering
\caption{Energy per layer, single data point. ArrowFlow \texttt{[128]} ($V{=}64$, $N{=}128$) vs.\ MLP ($64{\to}128$). All energies at 45nm CMOS.}
\label{tab:energy-layer}
\small
\begin{tabular}{lrrl}
\toprule
& Operations & Energy (pJ) & Op type \\
\midrule
\textbf{ArrowFlow sort layer} & & & \\
\quad Displacement (3NV) & 24,576  & 2,458 & INT32 @ 0.1\,pJ \\
\quad Argsort ($N\log_2 N$)& 896    & 90    & INT32 @ 0.1\,pJ \\
\quad \textit{Total} & 25,472  & \textbf{2,547} & \\
\midrule
\textbf{MLP layer} & & & \\
\quad Matrix multiply ($NV$ MACs) & 8,192 & 37,683 & FP32 @ 4.6\,pJ \\
\quad Bias + ReLU & 256 & 230 & FP32 \\
\quad \textit{Total} & 8,448 & \textbf{37,914} & \\
\midrule
\textbf{Ratio} & & \multicolumn{2}{l}{\textbf{14.9$\times$ in favor of ArrowFlow}} \\
\bottomrule
\end{tabular}
\end{table}

ArrowFlow performs $\sim$3$\times$ more operations but each costs $\sim$46$\times$ less energy, yielding a net \textbf{$\sim$15$\times$ arithmetic energy advantage} per layer.

\paragraph{Memory access energy.}
In practice, memory access dominates arithmetic cost in neural network inference \citep{sze2017efficient}. ArrowFlow's parameters are more compact: each filter is a permutation of $V$ items stored as 8-bit indices, requiring $NV$ bytes per layer. An MLP weight matrix requires $4NV$ bytes (FP32). For $N = 128$, $V = 64$:
\begin{itemize}[leftmargin=2em]
  \item ArrowFlow filter bank: $128 \times 64 = 8{,}192$ bytes (fits in 8KB SRAM)
  \item MLP weight matrix: $128 \times 64 \times 4 = 32{,}768$ bytes (requires 32KB)
\end{itemize}
When parameters fit in SRAM (5\,pJ per 32-bit read), the total memory energy for reading all weights once is 40,960\,pJ for the MLP vs.\ $\sim$10,240\,pJ for ArrowFlow's 8-bit indices---a 4$\times$ reduction. If parameters spill to DRAM (640\,pJ per read), the memory advantage becomes the dominant factor.

\paragraph{Full inference comparison.}
Table~\ref{tab:energy-full} compares full-inference energy for ArrowFlow's best Digits configuration (7-view ensemble, \texttt{[256]} layers, $V = 64$) against a single MLP ($64 \to 128 \to 10$):

\begin{table}[ht]
\centering
\caption{Full inference energy comparison on Digits. ArrowFlow: 7-view \texttt{[256]}, $V{=}64$. MLP: $64{\to}128{\to}10$. Arithmetic energy only (45nm CMOS).}
\label{tab:energy-full}
\small
\begin{tabular}{lrrr}
\toprule
Component & ArrowFlow (pJ) & MLP (pJ) \\
\midrule
Hidden, dist.+sort ($\times$7) & $7 (49{,}152{+}2{,}048)(0.1) = 35{,}840$ & $8{,}192 (4.6) = 37{,}683$ \\
Output, dist.+$\arg\min$ ($\times$7) & $7 \times 7{,}690 (0.1) = 5{,}383$ & $1{,}280 (4.6) = 5{,}888$ \\
Majority vote & $\sim$70  & --- \\
\midrule
\textbf{Total arithmetic} & \textbf{41,293} & \textbf{43,571} \\
\textbf{Ratio} & \multicolumn{2}{c}{$\approx 1.05\times$ in favor of ArrowFlow} \\
\bottomrule
\end{tabular}
\end{table}

Here the output layer compares the length-$256$ hidden ranking to $C=10$ class filters ($3CV=7{,}680$ displacement operations per view) and predicts by $\arg\min$ rather than a full sort; the hidden layer includes its $N\log_2 N$ ranking cost. The 7-view ensemble erodes nearly all of the per-operation advantage, reducing the ratio to only $\approx$1.05$\times$ at the full-inference arithmetic level. However, three factors make the real-world advantage substantially larger:
\begin{enumerate}[leftmargin=2em]
  \item \textbf{Memory bandwidth}: ArrowFlow's 8-bit parameters require 4$\times$ less memory bandwidth than FP32 weights. On memory-bound hardware (edge devices, neuromorphic chips), this translates directly to energy savings.
  \item \textbf{Parallelism}: The 7 views are embarrassingly parallel and operate on compact integer data, making them highly amenable to SIMD or multi-core execution with minimal inter-core communication.
  \item \textbf{No multiply hardware}: ArrowFlow requires only comparators, adders, and index tables---no floating-point multiply units. On an ASIC designed for ArrowFlow's operation profile, the die area and static power savings from eliminating FP multipliers would be substantial.
\end{enumerate}

\paragraph{Neuromorphic hardware alignment.}
ArrowFlow's computation maps naturally to neuromorphic architectures:

\begin{itemize}[leftmargin=2em]
  \item \textbf{Rank-order coding.} Thorpe's rank-order coding \citep{thorpe1998rapid, thorpe2001spike} encodes stimulus strength via the \emph{order} of neural spikes rather than firing rates---precisely the representation ArrowFlow operates on. A ranking filter in ArrowFlow corresponds to a stored spike-order template, and the footrule distance corresponds to the total spike-timing displacement. Imam and Cleland \citep{imam2020rapid} demonstrated that temporal-order-based classification on Intel Loihi achieves 1000$\times$ better energy efficiency than conventional CPU solutions.

  \item \textbf{Operation mapping.} ArrowFlow's displacement computation (integer subtraction + absolute value) maps to spike-timing difference circuits. The argsort operation maps to a sorting network---a fixed-topology circuit of compare-and-swap units. Batcher's bitonic sorting networks \citep{batcher1968sorting} require $O(N \log^2 N)$ comparators with a completely data-independent wiring pattern, ideal for fixed silicon. Each compare-and-swap costs ${\sim}0.15$\,pJ at 45nm, so sorting $N = 128$ distances requires ${\sim}100$\,pJ total---negligible compared to the displacement computation.

  \item \textbf{Event-driven sparsity.} Neuromorphic chips like Intel Loihi \citep{davies2018loihi} (23.6\,pJ/op at 14nm) and IBM TrueNorth \citep{merolla2014million} (26\,pJ/event at 28nm) achieve low energy through event-driven computation: neurons consume energy only when they spike. ArrowFlow's winner-take-all dynamics (Section~\ref{sec:arrow}, ND violation) create natural sparsity---most filters have large distances and only the closest few matter. An event-driven implementation could skip distance computation once a filter's partial distance exceeds the current minimum, reducing average-case energy well below the worst-case $3NV$ operations.
\end{itemize}

\paragraph{Scope and caveats.}
This analysis is deliberately narrow: it counts \emph{arithmetic} energy in the core sort layers only. Four caveats bound any system-level claim. (i)~The \emph{encoder} is not free---polynomial expansion, standardization, and especially the random/LDA projection use floating-point multiply--adds unless the projections are precomputed, quantized, or replaced by sparse/sign projections; our accounting excludes these. (ii)~A fair deployment baseline is often an \emph{INT8-quantized} MLP, not FP32, which narrows the per-operation gap substantially. (iii)~Classification may not require sorting all $N$ distances---nearest-filter or top-$k$ selection would lower ArrowFlow's cost further, but the comparison should hold the operation set fixed. (iv)~Memory access typically dominates arithmetic, so the realized advantage depends on the memory hierarchy and the degree of view parallelism. We therefore present the figures as an arithmetic-level indication, not a measured system-level speedup.

\paragraph{Summary.}
At the arithmetic-operation level, ArrowFlow's core layers replace FP32 MACs with integer comparisons and additions, a 15--150$\times$ per-operation gap that narrows at the full-inference level (to $\approx$1.05$\times$ for the seven-view Digits configuration). The system-level advantage will depend on the encoder, memory hierarchy, quantization baseline, and view parallelism. Even so, the combination of compact 8-bit parameters, no floating-point multiply units, and natural alignment with neuromorphic spike-order coding makes a custom ASIC or neuromorphic implementation of ArrowFlow's sort layers---comparators, integer adders, and index tables---a plausible route to low-energy inference, which we leave to future hardware work.

\subsection{The Argsort Bottleneck: A Fundamental Limit}

The argsort encoding is simultaneously ArrowFlow's greatest strength and its most fundamental limitation. By converting real-valued features to orderings, ArrowFlow gains scale invariance, noise robustness (Theorem~\ref{thm:stability}), and rank-only (magnitude-hiding) operation---but loses all magnitude information. Theorem~\ref{thm:capacity} quantifies this: the length-$e$ ordinal code has capacity $\log_2(e!) \approx e \log_2 e$ bits, finite and far below the infinite capacity of real-valued representations. Polynomial expansion reshapes the feature geometry feeding this fixed-length code rather than enlarging its capacity, and by Proposition~\ref{prop:poly-noise} it amplifies noise by a factor of $B^{k-1}$, weakening the stability bound. This tension---more information capacity vs.\ less noise robustness---is the formal basis of the polynomial trade-off documented in Section~\ref{sec:tradeoff}.

\subsection{Limitations}

\begin{itemize}[leftmargin=2em]
  \item \textbf{Computational cost:} Each forward pass is $O(NV)$ per data point. While the forward pass has been vectorized via batch distance computation, the backward pass remains sequential. This makes ArrowFlow approximately 10$\times$ slower than an equivalently-sized MLP per training iteration.
  \item \textbf{Information loss:} The argsort encoding discards magnitude information. Polynomial expansion mitigates but does not eliminate this, and introduces a robustness trade-off.
  \item \textbf{Scaling gap:} On larger datasets, ArrowFlow lags behind gradient-based methods. On MNIST (Section~\ref{sec:mnist-pca}), best ArrowFlow achieves 9.1\% vs.\ MLP's 4.2\%---a persistent $\sim$5pp gap. However, the gap narrows with wider architectures and has not plateaued, suggesting limited capacity is at least one bottleneck (though not necessarily the only one).
  \item \textbf{Natively ordinal data:} ArrowFlow's whole-permutation distance is too holistic when only local ordinal features are predictive (Sushi experiment, Section~\ref{sec:sushi}).
  \item \textbf{Rank-transform baseline:} On gene expression data (Section~\ref{sec:gene}), SVM on rank-transformed features achieves comparable or better batch-effect robustness than ArrowFlow native. The robustness advantage belongs to the ordinal representation principle, not to ArrowFlow specifically---though ArrowFlow provides this guarantee architecturally rather than as a preprocessing choice.
  \item \textbf{Hybrid architecture failure:} The tensor$\to$sort interface (Section~\ref{sec:hybrid}) destroys continuous representations via argsort discretization, preventing effective hybrid architectures.
  \item \textbf{Slow convergence:} The permutation-matrix accumulation makes conservative updates, requiring 200--500 iterations vs.\ $\sim$50 epochs for gradient-based methods.
\end{itemize}

\subsection{Future Directions}
\label{sec:future}

\begin{enumerate}[leftmargin=2em]
  \item \textbf{Batch vectorization:} The forward and backward passes process data points sequentially. Vectorizing via \texttt{scipy.cdist} or custom CUDA kernels could provide 10--50$\times$ speedup, making ArrowFlow practical for larger datasets.
  \item \textbf{Learned projections:} Replace random projection with end-to-end learned projections that minimize encoding loss while maintaining ensemble diversity. This could bridge the argsort bottleneck without abandoning the ordinal architecture.
  \item \textbf{Soft-sort encoding:} Using differentiable argsort approximations \citep{prillo2020softsort} to preserve partial magnitude information, potentially closing the gap on magnitude-sensitive datasets like Wine Quality.
  \item \textbf{Position-wise attention:} For natively ordinal data, a mechanism that weights individual item positions (rather than whole-permutation distances) could address the holistic-distance limitation exposed by the Sushi experiment.
  \item \textbf{Neuromorphic implementation:} ArrowFlow's integer-only arithmetic is 15--150$\times$ cheaper per operation than FP32 MACs (Section~\ref{sec:energy}). A custom ASIC implementing sort layers as comparator arrays and index tables could make inference dramatically cheaper than conventional networks, particularly on neuromorphic platforms \citep{davies2018loihi, merolla2014million} where spike-order coding is native.
  \item \textbf{Mode selection framework:} The polynomial trade-off (Section~\ref{sec:tradeoff}) suggests a practical deployment strategy: automatically select $\mathrm{pol\_deg}$=1 when the data is noisy, privacy-constrained, or incomplete; $\mathrm{pol\_deg}>1$ when clean accuracy is paramount.
  \item \textbf{Text and sequence domains:} The CountVectorizer nonzero encoding (converting word presence to native permutations) opens a path for applying ArrowFlow to text classification without the projection pipeline. Initial IMDB experiments are underway.
  \item \textbf{Broader combinatorial primitives:} Displacement-accumulation generalizes beyond permutations to other discrete structures---matchings, partial orders, lattice elements---potentially yielding a family of combinatorial learning architectures.
\end{enumerate}

\section{Conclusion}
\label{sec:conclusion}

ArrowFlow demonstrates that machine learning can operate effectively in the space of permutations. The architecture processes information through ranking filters---discrete permutations learned via motion accumulation rather than gradient descent---and achieves competitive classification accuracy through multi-view ensembles of diverse ordinal projections.

The key empirical findings, across 7 UCI datasets, MNIST, gene expression cancer classification, and preference data, with GridSearchCV-tuned baselines, are:
\begin{enumerate}[leftmargin=2em]
  \item ArrowFlow \textbf{matches or slightly edges every tuned baseline on Iris} (2.7\% vs.\ 3.3\%; a sub-sample margin, see Section~\ref{sec:experiments}) and is competitive on 5 of 7 UCI datasets, showing that zero-floating-point-parameter architectures can rival gradient-trained networks.
  \item The \textbf{polynomial degree acts as a master switch} between accuracy and robustness. At degree~1, ArrowFlow provides measurable noise robustness (8--28\% less degradation), near-zero-cost rank-only (magnitude-hiding) operation (+0.5pp under rank transforms), and superior missing-feature handling (lowest error on Iris at 50\% masking). At higher degrees, clean accuracy improves but perturbation amplification erases these advantages.
  \item On gene expression data, ordinal encoding provides \textbf{exact invariance to within-sample monotone batch effects}: SVM on raw values collapses from 0.6\% to 82.6\% under $\log$, $\sqrt{}$, or global scaling, while rank-based methods remain stable at 2.5\%. Under realistic per-gene batch effects ($\sigma \leq 0.5$), rank methods are unaffected while raw-input SVM error increases $28\times$. ArrowFlow embodies this ordinal robustness architecturally.
  \item On \textbf{MNIST via PCA}, ArrowFlow achieves 9.1\% error using only ordinal comparisons. It \textbf{scales strongly} with width and depth (18.7\%$\to$9.1\%), and the scaling curve has not plateaued, suggesting capacity is the bottleneck.
  \item The \textbf{multi-view ensemble} with diverse projections is the single most impactful component, reducing error by up to $\sim$3$\times$ (1.3--3.3$\times$ across datasets).
  \item \textbf{Freezing the output layer} consistently improves performance, a novel training strategy analogous to fixing classifier heads in transfer learning.
  \item ArrowFlow's \textbf{learning genuinely helps}: 2.5--12.5$\times$ improvement over kNN on the same encoded features proves that filter updates learn discriminative ordinal structure.
\end{enumerate}

Beyond empirics, the theoretical analysis (Section~\ref{sec:theory}) provides formal grounding: the argsort stability theorem (Theorem~\ref{thm:stability}) and its Gaussian corollary explain the noise robustness; the information capacity theorem (Theorem~\ref{thm:capacity}) quantifies the encoding bottleneck at $\log_2(e!)$ bits for the $e$-dimensional projected code; the polynomial noise amplification result (Proposition~\ref{prop:poly-noise}) explains the accuracy--robustness trade-off; and the accumulation consistency result (Proposition~\ref{prop:convergence}) guarantees that the learning rule converges to the mean-position ranking (the Borda aggregate of accumulated evidence).

The deeper contribution is conceptual: ArrowFlow shows that the violations of Arrow's social-choice axioms---context dependence, specialization, symmetry breaking---are not obstacles but \emph{mechanisms} for learning. Each ranking layer is a micro-democracy of filters whose deliberate ``unfairness'' creates nonlinearity, sparsity, and hierarchical composition. This reframes neural computation through the lens of combinatorics and voting theory, opening a new research direction at the intersection of social choice, rank-order coding, and representation learning.

ArrowFlow is an early entry in ordinal machine learning, not a finished one. The encoding bottleneck, the scaling gap, and the holistic-distance limitation are real constraints that bound the current system's practical applicability. But the fact that a system with \emph{zero floating-point parameters} in its core layers can reach 2.7\% error on Iris, 9.1\% on MNIST, and exact invariance to monotone batch effects on gene expression data---and that we can \emph{partly explain}, with proofs tied to the encoding and the learning rule, why its robustness properties hold---is, we believe, a foundation worth building on.

\bibliographystyle{plainnat}
\bibliography{references_arrowflow}

\end{document}